\documentclass{article}

\usepackage[final]{sods_2023}

\usepackage[utf8]{inputenc} 
\usepackage[T1]{fontenc}    
\usepackage{url}            
\usepackage{booktabs}       
\usepackage{amsfonts}       
\usepackage{nicefrac}       
\usepackage{microtype}      
\usepackage{mathtools}
\usepackage{wrapfig}
\usepackage{bm}
\usepackage{float}
\usepackage{amsthm}
\usepackage{amsmath}
\usepackage{amssymb}
\usepackage{booktabs}
\usepackage{graphicx}
\usepackage{graphbox}
\usepackage{notes2bib}
\usepackage{multirow}
\usepackage{algorithm, algorithmic}

\usepackage[colorlinks,linkcolor=blue]{hyperref}
\usepackage[dvipsnames]{xcolor}

\usepackage[capitalize,noabbrev]{cleveref}
\crefname{section}{Section}{Sections}
\crefname{theorem}{Theorem}{Theorems}
\crefname{lemma}{Lemma}{Lemmas}
\crefname{equation}{Equation}{Equations}
\crefname{proposition}{Proposition}{Propositions}
\crefname{claim}{Claim}{Claims}
\crefname{appendix}{Appendix}{Appendices}
\crefname{algorithm}{Algorithm}{Algorithms}
\crefname{figure}{Figure}{Figures}
\crefname{table}{Table}{Tables}
\crefname{remark}{Remark}{Remarks}
\crefname{definition}{Def.}{Definitions}
\crefname{corollary}{Corollary}{Corollaries}

\usepackage{bm}
\usepackage{float}
\usepackage{amsthm}
\usepackage{amsmath}
\usepackage{amssymb}
\usepackage{booktabs}
\usepackage{graphicx}
\usepackage{graphbox}
\usepackage{notes2bib}
\usepackage{subfigure}
\usepackage{multirow}
\usepackage{algorithm, algorithmic}

\usepackage[colorlinks,linkcolor=blue]{hyperref}
\usepackage[dvipsnames]{xcolor}
\definecolor{cite_color}{HTML}{114083}
\definecolor{link_color}{RGB}{0,102,102}
\definecolor{link_color}{RGB}{153, 0,0}  
\definecolor{url_color}{RGB}{153, 102,  0}
\definecolor{emp_color}{RGB}{0,0,255}
\hypersetup{
 colorlinks,
 citecolor=cite_color,
 linkcolor=link_color,
 urlcolor=url_color}
\graphicspath{{./figs/}}

\DeclarePairedDelimiterX{\infdivx}[2]{(}{)}{%
  #1\;\delimsize\|\;#2%
  }

\newcommand{\infdiv}{\operatorname{KL}\infdivx}

\def \x{\mathbf{x}}
\def \y{\mathbf{y}}

\def \z{\mathbf{z}}
\def \xib{\boldsymbol{\xi}}

\def \ED{\mathrm{ED}}
\def \data{\mathrm{data}}
\def \ebm{\mathrm{ebm}}
\def \e{\mathbf{e}}

\newtheorem{innercustomgeneric}{\customgenericname}
\providecommand{\customgenericname}{}
\newcommand{\newcustomtheorem}[2]{%
  \newenvironment{#1}[1]
  {%
   \renewcommand\customgenericname{#2}%
   \renewcommand\theinnercustomgeneric{##1}%
   \innercustomgeneric
  }
  {\endinnercustomgeneric}
}

\newcustomtheorem{customthm}{Theorem}

\DeclareMathOperator*{\argmin}{argmin}

\newtheorem{lemma}{Lemma}
\newtheorem{corollary}{Corollary}
\newtheorem{definition}{Definition}

\usepackage{thmtools}
\usepackage{thm-restate}

\usepackage{tikz}
\usetikzlibrary{decorations.pathreplacing,calc}

\newcommand\hinterval{0.23cm}
\newcommand\samplempwid{0.146}
\newcommand\samplehinterval{0.23cm}
\newcommand\samplefigwid{\textwidth}

\usepackage{etoc}
\etocdepthtag.toc{mtchapter}
\etocsettagdepth{mtchapter}{subsection}
\etocsettagdepth{mtappendix}{none}

\title{Training Discrete Energy-Based Models with \\ Energy Discrepancy}

\author{
Tobias Schröder$^1$\thanks{Correspondence to: Tobias Schröder, \texttt{t.schroeder21@imperial.ac.uk}}, \, Zijing Ou$^1$\thanks{Code: \url{https://github.com/J-zin/discrete-energy-discrepancy}, \texttt{z.ou22@imperial.ac.uk}}, \, Yingzhen Li$^1$, \, Andrew Duncan$^{1,2}$\\
$^1$ Imperial College London, UK, $^2$ The Alan Turing Institute, UK\\
\texttt{\{t.schroeder21, z.ou22, yingzhen.li, a.duncan\}@imperial.ac.uk}\\
\setcounter{footnote}{2}
}

\begin{document}

\maketitle
\begin{abstract}
Training energy-based models (EBMs) on discrete spaces is challenging because sampling over such spaces can be difficult. We propose to train discrete EBMs with energy discrepancy (ED), a novel type of contrastive loss functional which only requires the evaluation of the energy function at data points and their perturbed counter parts, thus not relying on sampling strategies like Markov chain Monte Carlo (MCMC). Energy discrepancy offers theoretical guarantees for a broad class of perturbation processes of which we investigate three types: perturbations based on Bernoulli noise, based on deterministic transforms, and based on neighbourhood structures. We demonstrate their relative performance on lattice Ising models, binary synthetic data, and discrete image data sets.
\end{abstract}

\section{Introduction}
Building large-scale probabilistic models for discrete data is a critical challenge in machine learning for its broad applicability to perform inference and generation tasks on images, text, or graphs. Energy-based models (EBMs) are a class of particularly flexible models $p_\ebm \propto \exp(-U)$, where the modelling of the energy function $U$ through a neural network function can be taylored to the data set of interest. However, EBMs are notoriously difficult to train due to the intractability of their normalisation.

\begin{wrapfigure}{r}{0.20\linewidth}
\vspace{-4mm}
\centering
\includegraphics[width=.2\textwidth]{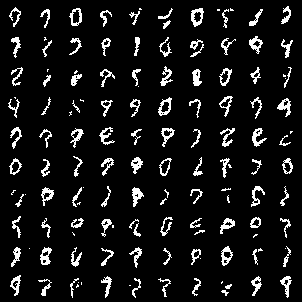}
\vspace{-6mm}
\caption{Generated samples from the EBM trained with Energy Discrepancy on static MNIST.}
\vspace{-4mm}
\label{fig:static_mnist_ed_bern}
\end{wrapfigure}
The most popular paradigm for the training of EBMs is the contrastive divergence (CD) algorithm \citep{hinton2002training} which performs approximate maximum likelihood estimation by using short-run Markov Chain Monte Carlo (MCMC) to approximate intractable expectations with respect to $p_\ebm$. The success of CD has lead to rich research results on sampling from discrete distributions to enable fast and accurate estimation of the EBM \citep{zanella2020informed,grathwohl2021oops,zhang2022langevin,sun2022path,sun2022optimalscaling,sun2023discrete,emami2023plug}.
However, training EBMs with CD remains challenging: Firstly, discrete probabilistic models often exhibit a large number of spurious modes which are difficult to explore even for the most advanced sampling algorithms. Secondly, CD lacks theoretical guarantees due to short run MCMC \citep{carreira2005contrastive} and often times leads to malformed energy landscapes \citep{nijkamp2019learning}.

We propose the usage of a new type of loss function called Energy Discrepancy (ED) \citep{schroeder2023energy} for the training of energy-based models on discrete spaces. The definition of ED only requires the evaluation of the EBM on positive and contrasting, negative samples. Unlike CD, energy discrepancy does not require sampling from the model during training, thus allowing for fast training with theoretical guarantees.
We demonstrate the effectiveness of ED by training Ising models, estimating discrete densities, and modelling discrete images in high-dimensions (see \cref{fig:static_mnist_ed_bern} for an illustration).
\section{Energy Discrepancies}
Energy discrepancies are based on the idea that if information is processed through a channel $\mathcal Q$ then information will be lost. Mathematically, this is expressed through the data processing inequality $\infdiv{Qp_\data}{Qp_\ebm}\leq \infdiv{p_\data}{p_\ebm}$. Consequently, the difference of the two KL divergences forms a valid loss for density estimation \citep{Lyu2011KLContractions}. Retaining only terms that depend on the energy function $U$ results in the energy discrepancy \citep{schroeder2023energy}:

{\definition[Energy Discrepancy]\label{definition-energy-discrepancy}{
Let $p_{\mathrm{data}}$ be a positive density on a measure space ($\mathcal{X}$, $\mathrm d\x$)\footnotemark and let $q(\y|\x)$ be a conditional probability density. Define the \emph{contrastive potential} induced by $q$ as
{\begin{align} \label{definition-contrastive-potential}
    U_q (\y) := - \log \sum_{\x'\in\mathcal X} q(\y | \x') \exp(-U(\x'))
\end{align}
\footnotetext{On discrete spaces $\mathrm d\x$ is assumed to be a counting measure. On continuous spaces $\mathcal X$, the appearing sums and expectations turn into integrals with respect to the Lebesgue measure}}
We define the \emph{energy discrepancy} between $p_{\mathrm{data}}$ and $U$ induced by $q$ as
\begin{equation*} \label{energy-discrepancy-loss}
    \ED_q (p_{\mathrm{data}}, U) := \mathbb{E}_{p_{\mathrm{data}}(\x)} [U(\x)] - \mathbb{E}_{p_{\mathrm{data}}(\x)}\mathbb{E}_{q(\y | \x)} [U_q (\y)]. 
\end{equation*}
}}

The validity of this loss functional is given by the following non-parametric estimation result, previously stated in \citet{schroeder2023energy}:
\begin{restatable}[]{theorem}{restatheoremone}
\label{theorem-energy-discrepancy}
Let $p_{\mathrm{data}}$ be a positive probability density on ($\mathcal{X}$, $\mathrm{d}\x$). Assume that for all $\x\sim p_\data$ and $\y\sim q(\y\vert\x)$, $\mathrm{Var}(\x \vert \y)>0$. Then, the energy discrepancy $\ED_q$ is functionally convex in $U$ and has, up to additive constants, a unique global minimiser $U^* = \argmin \ED_q (p_{\mathrm{data}}, U)$. Furthermore, this minimiser is the Gibbs potential for the data distribution, \it{i.e.} $p_\data \propto \exp(-U^\ast)$.
\end{restatable}

We give the proof of \cref{theorem-energy-discrepancy} in \cref{appendix-proof-theorem-energy-discrepancy}. The perturbation $q$ can be chosen quite generally as long as it can be guaranteed that computing $\y$ comes at a loss of information which mathematically is expressed through the variance of recovering $\x$ from $\y\sim q(\y\vert \x)$ being positive. In the next section, we propose some practical choices for $q$.
\subsection{Training Discrete Energy-Based Models with Energy Discrepancy}
The perturbation process $q$ needs to be chosen under the following considerations: 1) The contrastive potential $U_q(\y)$ has a numerically tractable approximation. 2) The negative samples obtained through $q$ are informative for training the EBM when only finite amounts of data are available. We propose three categories for constructing perturbative processes:

\textbf{Bernoulli Perturbation.}
For $\varepsilon \in (0, 1)$, let $\xib \sim \mathrm{Bernoulli}(\varepsilon)^d$. On $\mathcal X= \{0, 1\}^d$, consider the perturbation $\y = \x + \xib \,\mathrm{mod}(2)$ which induces a symmetric transition density $q(\y-\x)$ on $\{0, 1\}^d$. Due to the symmetry of $q$, we can then write the contrastive potential as
\begin{equation*}
    U_{\mathrm{bernoulli}}(\y) = -\log\sum_{\x'\in \mathcal X} q(\y-\x') \exp(-U(\x')) = -\log \mathbb E_{\x'\sim q(\y-\x')}[\exp(-U(\x'))]
\end{equation*}
The expectation on the right hand side can now be approximated via sampling $M$ Bernoulli random variables $\xib^j$ and taking the remainder of $(\y+\xib^j)/2$. We denote this method as ED-Bern.

\textbf{Deterministic Transformation.}
The perturbation $q$ can also be defined through a deterministic information loosing map $g:\mathcal X \to \mathcal Y$, where the space $\mathcal Y$ may or may not be equal to $\mathcal X$ depending on the choice of $g$. The contrastive potential can be expressed in terms of the preimage of $g$, i.e.
\begin{align*}
    U_g(\y) = -\log \sum_{\{\x': g(\x') = \y\}} \exp(-U(\x')) = -\log \mathbb E_{\x' \sim \mathcal U(\{g^{-1}(\y)\})}[\exp(-U(\x'))] - c
\end{align*}
with  $c = \log \vert \{g^{-1}(\y)\} \vert$. Again, the contrastive potential can be approximated through sampling $M$ instances from the uniform distribution over the set $\{\x': g(\x') = \y\}$. In our numerical experiments, we focus on the mean-pooling transform $g_{\mathrm{pool}}$ whose inverse are block-wise permutations. For details, see \cref{subsct:mean-pooling-appendix}. We denote this method as ED-Pool.

\textbf{Neighbourhood-based Transformation.}
Finally, inspired from concrete score matching \citep{Meng2022concreteSM}, we may define energy discrepancies based on neighbourhood maps $\x\mapsto \mathcal N(\x)\in \mathcal X^K$ which assign each point $\x\in \mathcal X$ a set of $K$ neighbours\footnote{We are making the assumption that the numbers of neighbours is the same for each point. A more general case is discussed in \cref{appendix-sec-directed-neighbour-structures}.}. We define the forward perturbation $q(\y\vert\x)$ by selecting neighbours $\y\sim \mathcal U(\mathcal N(\x))$ uniformly at random. Conversely, the contrastive potential can be expressed in terms of the inverse neighbourhood $\y\mapsto \mathcal N^{-1}(\y)\in\mathcal X^{K}$, i.e. the set of points that have $\y$ to their neighbour. We then obtain for the contrastive potential
\begin{align*}
    U_{\mathcal N}(\y) = -\log \frac{1}{K}\sum_{\x'\in\mathcal X: \y\in\mathcal N(\x')} \exp(-U(\x')) = -\log \mathbb E_{\x' \sim \mathcal U(\{\mathcal N^{-1}(\y)\})}[\exp(-U(\x'))]\,.
\end{align*}
In practice, we choose the grid neighbourhood (\cref{subsct:grid-neighbourhood-appendix}) and denote this method by ED-Grid.

\textbf{Stabilising Training.}
Above schemes permit the approximation of the contrastive potential from $M$ samples which are generated by first sampling $\y\sim q(\y\vert \x)$, after which we compute $M$ approximate recoveries $\x_-^j$. The full loss can then be constructed for each data point $\x_+\sim p_\data$ by calculating $\log \sum_{j= 1}^M \exp(U(\x_+) - U(\x_-^j))-\log(M)$ using the numerically stabilised logsumexp function. In practice, however, we find that this estimator for energy discrepancy is biased due to the logarithm and can exhibit high variance. To stabilise training, we introduce an offset for the logarithm which introduces a deterministic lower bound for the loss. This yields the energy discrepancy loss function
\begin{equation}\label{equ:stabilised-loss-function}
    \mathcal L_{q, M, w}(U) := \frac{1}{N} \sum_{i=1}^N \log\left(w+ \sum_{j= 1}^M \exp(U(\x^i_+) - U(\x_-^{i,j}))\right) - \log(M)
\end{equation}
with $\x^i_+ \sim p_\data$. In \cref{subsct:consistency-proof-appendix} we proof that this approximation is consistent for any fixed $w$:

\begin{restatable}[]{theorem}{restatetheoremtwo}
For every $\varepsilon \!>\! 0$ there exist $N, M \!\in\! \mathbb N$ such that $\lvert\mathcal L_{q, M, w}(U) \!-\! \ED_q(p_\data, U)\rvert \!<\! \varepsilon$ a.s..
\end{restatable}
\section{Experiments}

\begin{table*}[t]
\small
\setlength{\tabcolsep}{1.5mm}
\centering
\caption{
Experiment results with seven 2D synthetic problems.
We display the negative log-likelihood (NLL) and MMD (in units of $1\times 10^{-4}$). The results of baselines are taken from \cite{zhang2022generative}.
}
\label{tab:synthetic_nll_mmd}
\begin{tabular}{c|l|ccccccc}
\toprule
Metric & Method & 2spirals & 8gaussians & circles & moons & pinwheel & swissroll & checkerboard \\
\midrule
\multirow{6}{*}{NLL$\downarrow$} 
& PCD    &  $20.094$ & $19.991$ &$20.565$&$19.763$&$19.593$&$20.172$&$21.214$ \\
& ALOE+  &  $20.062$ & $19.984$ & $20.570$ & $19.743$ & $19.576$ & $20.170$ & $21.142$\\
& EB-GFN & ${20.050}$ & ${19.982}$ & $\textbf{20.546}$ & ${19.732}$ & ${19.554}$ & ${20.146}$ & ${20.696}$ \\
& ED-Bern (ours) & $\textbf{20.039}$ & ${19.992}$ & ${20.601}$ & $\textbf{19.710}$ & ${19.568}$ & $\textbf{20.084}$ & ${20.679}$ \\
& ED-Pool (ours) & ${20.051}$ & ${19.999}$ & ${20.604}$ & ${19.721}$ & $\textbf{19.531}$ & $\textbf{20.084}$ & $\textbf{20.676}$ \\
& ED-Grid (ours) & ${20.049}$ & $\textbf{19.965}$ & ${20.601}$ & ${19.715}$ & ${19.564}$ & ${20.088}$ & ${20.678}$ \\
\midrule
\multirow{6}{*}{MMD$\downarrow$} 
& PCD    &  $2.160$&$0.954$&$0.188$&$0.962$&$0.505$&$1.382$& $2.831$ \\
& ALOE+  &  $ {0.149} $ & ${0.078}$ & $0.636$ & $0.516$ & $1.746$ & $0.718$ & $12.138$ \\
& EB-GFN &  $0.583$ & $0.531$ & ${0.305}$ & ${0.121}$ & ${0.492}$ & ${0.274}$ & $\textbf{1.206}$\\ 
& ED-Bern (ours) & ${0.120}$ & ${0.014}$ & ${0.137}$ & $\textbf{-0.088}$ & $\textbf{0.046}$ & $\textbf{0.045}$ & ${1.541}$ \\
& ED-Pool (ours) & ${0.129}$ & ${\text{-}0.003}$ & $\textbf{-0.021}$ & ${0.042}$ & ${0.126}$ & ${0.101}$ & ${2.080}$ \\
& ED-Grid (ours) & $\textbf{0.097}$ & $\textbf{-0.066}$ & ${0.022}$ & ${0.018}$ & ${0.351}$ & ${0.097}$ & ${2.049}$ \\
\bottomrule
\end{tabular}
\vspace{-5mm}
\end{table*}

\begin{wrapfigure}{r}{0.60\linewidth}
\vspace{-4mm}
\centering
\includegraphics[width=.14\textwidth]{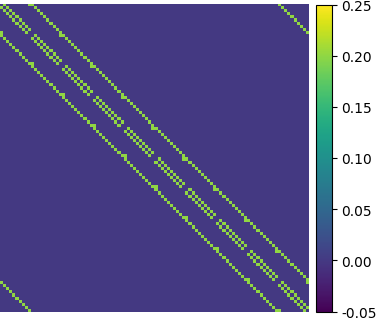}
\includegraphics[width=.14\textwidth]{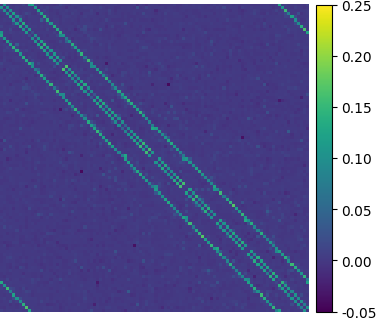}
\includegraphics[width=.14\textwidth]{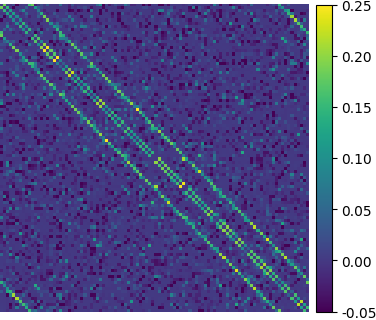}
\includegraphics[width=.14\textwidth]{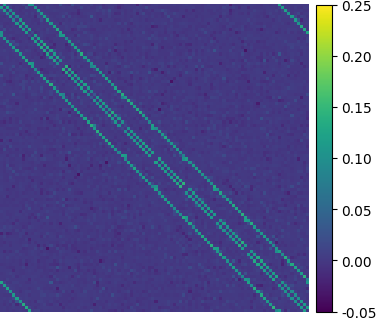}
\vspace{-2mm}
\caption{Experiment results on learning lattice Ising models. Left to right: ground truth, ED-Bern, ED-Pool, ED-Grid.}
\label{fig:ising_learned_matrix}
\end{wrapfigure}

\textbf{Training Ising Models.}
We evaluate the proposed methods on the lattice Ising model, which has the form of
\begin{align}
    p(\x) \propto \exp (\x^T J \x), \  \x \in \{-1,1\}^D, \nonumber
\end{align}
where $J=\sigma A_D$ with $\sigma \in \mathbb{R}$ and $A_D$ being the adjacency matrix of a $D\times D$ grid.
Following \cite{zhang2022generative}, we generate training data through Gibbs sampling and use the generated data to fit a symmetric matrix $J$ via energy discrepancy. 
In \cref{fig:ising_learned_matrix}, we consider $D=10\times 10$ grids with $\sigma =0.2$ and illustrate the learned matrix $J$ using a heatmap. It can be seen that the variants of energy discrepancy can identify the pattern of the ground truth, confirming the effectiveness of our methods. We defer experimental details and quantitative results comparing with baselines to \cref{appendix-sec-training-ising-models}.

\textbf{Discrete Density Estimation.}
In this experiment, we follow the experimental setting of \cite{dai2020learning,zhang2022generative}, which aims to model discrete densities over $32$-dimensional binary data that are discretisations of continuous densities on the plane (see \cref{fig:true_synthetic_samples}). Specifically, we convert each planar data point $\hat{\x} \in \mathbb{R}^2$ to a binary data point $\x \in \{0,1\}^{32}$ via Gray code \citep{gray1953pulse}. Consequently, the models face the challenge of modeling data in a discrete space, which is particularly difficult due to the non-linear transformation from $\hat{\x}$ to $\x$.

We compare our methods to three baselines: PCD \citep{tieleman2008training}, ALOE+ \citep{dai2020learning}, and EB-GFN \citep{zhang2022generative}. The experimental details are given in \cref{appendix-sec-discrete-density-estimation}. 
For qualitative evaluation, we visualise the energy landscapes learned by our methods in \cref{fig:toy_result_visualisation}. It shows that energy discrepancy is able to faithfully model multi-modal distributions and accurately learn the sharp edges present in the data support. For further qualitative comparisons, we refer to the energy landscapes of baseline methods presented in Figure C.2 of \cite{zhang2022generative}.
Moreover, we quantitatively evaluate different methods in \cref{tab:synthetic_nll_mmd} by showing the negative log-likelihood (NLL) and the exponential Hamming MMD \citep{gretton2012kernel}. Perhaps surprisingly, we find that energy discrepancy outperforms the baselines on most settings, despite not requiring MCMC simulation like PCD or training an additional variational network like ALOE and EB-GFN. A possible explanation for this are biases introduced by short-run MCMC sampling in the case of PCD or non-converged variational proposals in ALOE. By definition, ED transforms the data distribution as well as the energy function which corrects for such biases.

\textbf{Discrete Image Modelling.}
Here, we evaluate our methods in discrete high-dimensional spaces. Following the settings in \cite{grathwohl2021oops,zhang2022langevin}, we conduct experiments on four different binary image datasets. Training details are given in \cref{appendix-sec-discrete-image-modelling}. After training, we adopt Annealed Importance Sampling \citep{neal2001annealed} to estimate the log-likelihoood.

The baselines include persistent contrastive divergence with vanilla Gibbs sampling, Gibbs-With-Gradient \cite[GWG]{grathwohl2021oops}, Generative-Flow-Network \cite[GFN]{zhang2022generative}, and Discrete-Unadjusted-Langevin-Algorithm \cite[DULA]{zhang2022langevin}. The NLLs on the test set are reported in \cref{tab:image_logll}. We see that energy discrepancy yields comparable performances to the baselines, while ED-Pool is unable to capture the data distribution. We emphasise that energy discrepancy only requires $M$ (here, $M=32$) evaluations of the energy function per data point in parallel. This is notably fewer than contrastive divergence, which requires simulating multiple MCMC steps without parallelisation. 
We also visualise the generated samples in \cref{fig:sample-ebm-appendix}, which showcase the diversity and high quality of the images generated by ED-Bern and ED-Grid.
However, we observed that ED-Pool suffers from mode collapse.

\begin{table}[t]
\vspace{-3mm}
\small
\setlength{\tabcolsep}{1.1mm}
\centering
\caption{Experimental results on discrete image modelling. We report the negative log-likelihood (NLL) on the test set for different models. The results of Gibbs, GWG, and DULA are taken from \cite{zhang2022langevin}, and the result of EB-GFN is from \cite{zhang2022generative}.}
\label{tab:image_logll}
\vspace{-2mm}
\begin{tabular}{l|ccccccc}
\toprule 
Dataset $\backslash$ Method & Gibbs & GWG & EB-GFN & DULA & ED-Bern (ours) & ED-Pool (ours) & ED-Grid (ours) \\
\midrule 
Static MNIST & $117.17$ & $\textbf{80.01}$ & $102.43$ & $80.71$ & 95.38 & 168.07 & 90.15 \\
Dynamic MNIST & $121.19$ & $\textbf{80.51}$ & ${105.75}$ & $81.29$ & 97.03 & 144.26 & 81.26 \\
Omniglot & $142.06$ & ${94.72}$ & ${112.59}$ & $145.68$ & 97.87 & 118.66 & \textbf{94.64} \\
Caltech Silhouettes & $163.50$ & $\textbf{96.20}$ &  ${185.57}$ & $100.52$ & 96.36 & 501.96 & 117.70 \\
\bottomrule
\end{tabular}
\vspace{-5mm}
\end{table}
\section{Conclusion and Outlook}
In this paper we demonstrate how energy discrepancy can be used for efficient and competitive training of energy-based models on discrete data without MCMC. The loss can be defined based on a large class of perturbative processes of which we introduce three types: noise, determinstic transform, and neighbourhood-based transform. Our results show that the choice of perturbation matters and motivates further research on effective choices depending on the data structure of interest.

We observe empirically that similarly to other contrastive losses, energy discrepancy shows limitations when the ambient dimension of $\mathcal X$ is significantly larger than the intrinsic dimension of the data. In these cases, training is aided significantly by a base distribution that models the lower-dimensional space populated by data. For this reason, the adoption of ED on new data sets or different data structures may require adjustments to the methodology such as learning appropriate base distributions and finding more informative perturbative transforms.

For future work, we are interested in how this work extends to highly structured data such as graphs or text. These settings may require a deeper understanding of how the perturbation influences the performance of ED and what is gained from gradient information in CD \citep{zhang2022langevin,grathwohl2021oops} or ratio matching \citep{liu2023RMwGGIS}.

\section*{Acknowledgements}
TS would like to thank G.A. Pavliotis for insightful discussions leading up to the presented work. TS was supported by the EPSRC-DTP scholarship partially funded by the Ddepartment of Mathematics, Imperial College London. ZO was supported by the Lee Family Scholarship. ABD was supported by Wave 1 of The UKRI Strategic Priorities Fund under the EPSRC Grant EP/T001569/1 and EPSRC Grant EP/W006022/1, particularly the “Ecosystems of Digital Twins” theme within those grants and The Alan Turing Institute. We thank the anonymous reviewer for their comments.

{
\bibliography{main}

\begin{thebibliography}{33}
\providecommand{\natexlab}[1]{#1}
\providecommand{\url}[1]{\texttt{#1}}
\expandafter\ifx\csname urlstyle\endcsname\relax
  \providecommand{\doi}[1]{doi: #1}\else
  \providecommand{\doi}{doi: \begingroup \urlstyle{rm}\Url}\fi

\bibitem[Carreira-Perpinan \& Hinton(2005)Carreira-Perpinan and
  Hinton]{carreira2005contrastive}
Carreira-Perpinan, M.~A. and Hinton, G.
\newblock On contrastive divergence learning.
\newblock In \emph{International workshop on artificial intelligence and
  statistics}, pp.\  33--40. PMLR, 2005.

\bibitem[Ceylan \& Gutmann(2018)Ceylan and
  Gutmann]{ConditionalNoiseContrastiveEstimation}
Ceylan, C. and Gutmann, M.~U.
\newblock Conditional noise-contrastive estimation of unnormalised models.
\newblock In Dy, J. and Krause, A. (eds.), \emph{Proceedings of the 35th
  International Conference on Machine Learning}, volume~80 of \emph{Proceedings
  of Machine Learning Research}, pp.\  726--734. PMLR, 10--15 Jul 2018.
\newblock URL \url{https://proceedings.mlr.press/v80/ceylan18a.html}.

\bibitem[Dai et~al.(2020)Dai, Singh, Dai, Sutton, and
  Schuurmans]{dai2020learning}
Dai, H., Singh, R., Dai, B., Sutton, C., and Schuurmans, D.
\newblock Learning discrete energy-based models via auxiliary-variable local
  exploration.
\newblock \emph{Advances in Neural Information Processing Systems},
  33:\penalty0 10443--10455, 2020.

\bibitem[Eikema et~al.(2022)Eikema, Kruszewski, Dance, Elsahar, and
  Dymetman]{eikemaapproximate}
Eikema, B., Kruszewski, G., Dance, C.~R., Elsahar, H., and Dymetman, M.
\newblock An approximate sampler for energy-based models with divergence
  diagnostics.
\newblock \emph{Transactions of Machine Learning Research}, 2022.

\bibitem[Emami et~al.(2023)Emami, Perreault, Law, Biagioni, and
  John]{emami2023plug}
Emami, P., Perreault, A., Law, J., Biagioni, D., and John, P.~S.
\newblock Plug \& play directed evolution of proteins with gradient-based
  discrete {MCMC}.
\newblock \emph{Machine Learning: Science and Technology}, 4\penalty0
  (2):\penalty0 025014, 2023.

\bibitem[Grathwohl et~al.(2021)Grathwohl, Swersky, Hashemi, Duvenaud, and
  Maddison]{grathwohl2021oops}
Grathwohl, W., Swersky, K., Hashemi, M., Duvenaud, D., and Maddison, C.~J.
\newblock Oops {I} took a gradient: Scalable sampling for discrete
  distributions.
\newblock \emph{arXiv preprint arXiv:2102.04509}, 2021.

\bibitem[Gray(1953)]{gray1953pulse}
Gray, F.
\newblock Pulse code communication.
\newblock \emph{United States Patent Number 2632058}, 1953.

\bibitem[Gretton et~al.(2012)Gretton, Borgwardt, Rasch, Sch{\"o}lkopf, and
  Smola]{gretton2012kernel}
Gretton, A., Borgwardt, K.~M., Rasch, M.~J., Sch{\"o}lkopf, B., and Smola, A.
\newblock A kernel two-sample test.
\newblock \emph{The Journal of Machine Learning Research}, 13\penalty0
  (1):\penalty0 723--773, 2012.

\bibitem[Gutmann \& Hyv{\"a}rinen(2010)Gutmann and
  Hyv{\"a}rinen]{gutmann2010noise}
Gutmann, M. and Hyv{\"a}rinen, A.
\newblock Noise-contrastive estimation: A new estimation principle for
  unnormalized statistical models.
\newblock In \emph{Proceedings of the thirteenth international conference on
  artificial intelligence and statistics}, pp.\  297--304. JMLR Workshop and
  Conference Proceedings, 2010.

\bibitem[He et~al.(2016)He, Zhang, Ren, and Sun]{he2016deep}
He, K., Zhang, X., Ren, S., and Sun, J.
\newblock Deep residual learning for image recognition.
\newblock In \emph{Proceedings of the IEEE conference on computer vision and
  pattern recognition}, pp.\  770--778, 2016.

\bibitem[Hinton(2002)]{hinton2002training}
Hinton, G.~E.
\newblock Training products of experts by minimizing contrastive divergence.
\newblock \emph{Neural computation}, 14\penalty0 (8):\penalty0 1771--1800,
  2002.

\bibitem[Hyv{\"a}rinen(2007)]{hyvarinen2007some}
Hyv{\"a}rinen, A.
\newblock Some extensions of score matching.
\newblock \emph{Computational statistics \& data analysis}, 51\penalty0
  (5):\penalty0 2499--2512, 2007.

\bibitem[Kingma \& Ba(2014)Kingma and Ba]{kingma2014adam}
Kingma, D.~P. and Ba, J.
\newblock Adam: A method for stochastic optimization.
\newblock \emph{arXiv preprint arXiv:1412.6980}, 2014.

\bibitem[Lazaro-Gredilla et~al.(2021)Lazaro-Gredilla, Dedieu, and
  George]{lazaro2021perturb}
Lazaro-Gredilla, M., Dedieu, A., and George, D.
\newblock Perturb-and-max-product: Sampling and learning in discrete
  energy-based models.
\newblock \emph{Advances in Neural Information Processing Systems},
  34:\penalty0 928--940, 2021.

\bibitem[Liu et~al.(2023)Liu, Liu, and Ji]{liu2023RMwGGIS}
Liu, M., Liu, H., and Ji, S.
\newblock Gradient-guided importance sampling for learning binary energy-based
  models.
\newblock 2023.

\bibitem[Lyu(2011)]{Lyu2011KLContractions}
Lyu, S.
\newblock Unifying non-maximum likelihood learning objectives with minimum {KL}
  contraction.
\newblock In Shawe-Taylor, J., Zemel, R., Bartlett, P., Pereira, F., and
  Weinberger, K. (eds.), \emph{Advances in Neural Information Processing
  Systems}, volume~24. Curran Associates, Inc., 2011.
\newblock URL
  \url{https://proceedings.neurips.cc/paper/2011/file/a3f390d88e4c41f2747bfa2f1b5f87db-Paper.pdf}.

\bibitem[Lyu(2012)]{lyu2012interpretation}
Lyu, S.
\newblock Interpretation and generalization of score matching.
\newblock \emph{arXiv preprint arXiv:1205.2629}, 2012.

\bibitem[Majerek et~al.(2005)Majerek, Nowak, and Ziba]{ConditionalLLN}
Majerek, D., Nowak, W., and Ziba, W.
\newblock Conditional strong law of large number.
\newblock \emph{International Journal of Pure and Applied Mathematics}, 20, 01
  2005.

\bibitem[Meng et~al.(2022)Meng, Choi, Song, and Ermon]{Meng2022concreteSM}
Meng, C., Choi, K., Song, J., and Ermon, S.
\newblock Concrete score matching: Generalized score matching for discrete
  data.
\newblock In Koyejo, S., Mohamed, S., Agarwal, A., Belgrave, D., Cho, K., and
  Oh, A. (eds.), \emph{Advances in Neural Information Processing Systems},
  volume~35, pp.\  34532--34545. Curran Associates, Inc., 2022.
\newblock URL
  \url{https://proceedings.neurips.cc/paper_files/paper/2022/file/df04a35d907e894d59d4eab1f92bc87b-Paper-Conference.pdf}.

\bibitem[Neal(2001)]{neal2001annealed}
Neal, R.~M.
\newblock Annealed importance sampling.
\newblock \emph{Statistics and computing}, 11:\penalty0 125--139, 2001.

\bibitem[Nijkamp et~al.(2019)Nijkamp, Hill, Zhu, and Wu]{nijkamp2019learning}
Nijkamp, E., Hill, M., Zhu, S.-C., and Wu, Y.~N.
\newblock Learning non-convergent non-persistent short-run {MCMC} toward
  energy-based model.
\newblock \emph{arXiv preprint arXiv:1904.09770}, 2019.

\bibitem[Papandreou \& Yuille(2011)Papandreou and
  Yuille]{papandreou2011perturb}
Papandreou, G. and Yuille, A.~L.
\newblock Perturb-and-map random fields: Using discrete optimization to learn
  and sample from energy models.
\newblock In \emph{2011 International Conference on Computer Vision}, pp.\
  193--200. IEEE, 2011.

\bibitem[Ramachandran et~al.(2017)Ramachandran, Zoph, and
  Le]{ramachandran2017searching}
Ramachandran, P., Zoph, B., and Le, Q.~V.
\newblock Searching for activation functions.
\newblock \emph{arXiv preprint arXiv:1710.05941}, 2017.

\bibitem[Rhodes \& Gutmann(2022)Rhodes and Gutmann]{rhodes2022enhanced}
Rhodes, B. and Gutmann, M.
\newblock Enhanced gradient-based {MCMC} in discrete spaces.
\newblock \emph{arXiv preprint arXiv:2208.00040}, 2022.

\bibitem[Schröder et~al.(2023)Schröder, Ou, Lim, Li, Vollmer, and
  Duncan]{schroeder2023energy}
Schröder, T., Ou, Z., Lim, J.~N., Li, Y., Vollmer, S.~J., and Duncan, A.~B.
\newblock Energy discrepancies: A score-independent loss for energy-based
  models, 2023.
\newblock URL \url{https://arxiv.org/abs/2307.06431}.

\bibitem[Sun et~al.(2022{\natexlab{a}})Sun, Dai, and
  Schuurmans]{sun2022optimalscaling}
Sun, H., Dai, H., and Schuurmans, D.
\newblock Optimal scaling for locally balanced proposals in discrete spaces.
\newblock In Koyejo, S., Mohamed, S., Agarwal, A., Belgrave, D., Cho, K., and
  Oh, A. (eds.), \emph{Advances in Neural Information Processing Systems},
  volume~35, pp.\  23867--23880. Curran Associates, Inc., 2022{\natexlab{a}}.
\newblock URL
  \url{https://proceedings.neurips.cc/paper_files/paper/2022/file/96c6f409a374b5c81d2efa4bc5526f27-Paper-Conference.pdf}.

\bibitem[Sun et~al.(2022{\natexlab{b}})Sun, Dai, Xia, and
  Ramamurthy]{sun2022path}
Sun, H., Dai, H., Xia, W., and Ramamurthy, A.
\newblock Path auxiliary proposal for {MCMC} in discrete space.
\newblock In \emph{International Conference on Learning Representations},
  2022{\natexlab{b}}.

\bibitem[Sun et~al.(2023)Sun, Dai, Dai, Zhou, and Schuurmans]{sun2023discrete}
Sun, H., Dai, H., Dai, B., Zhou, H., and Schuurmans, D.
\newblock Discrete {Langevin} samplers via {Wasserstein} gradient flow.
\newblock In \emph{International Conference on Artificial Intelligence and
  Statistics}, pp.\  6290--6313. PMLR, 2023.

\bibitem[Tieleman(2008)]{tieleman2008training}
Tieleman, T.
\newblock Training restricted boltzmann machines using approximations to the
  likelihood gradient.
\newblock In \emph{Proceedings of the 25th international conference on Machine
  learning}, pp.\  1064--1071, 2008.

\bibitem[van~den Oord et~al.(2018)van~den Oord, Li, and
  Vinyals]{Oord2018RepresentationLearning}
van~den Oord, A., Li, Y., and Vinyals, O.
\newblock Representation learning with contrastive predictive coding.
\newblock \emph{CoRR}, abs/1807.03748, 2018.
\newblock URL \url{http://arxiv.org/abs/1807.03748}.

\bibitem[Zanella(2020)]{zanella2020informed}
Zanella, G.
\newblock Informed proposals for local {MCMC} in discrete spaces.
\newblock \emph{Journal of the American Statistical Association}, 115\penalty0
  (530):\penalty0 852--865, 2020.

\bibitem[Zhang et~al.(2022{\natexlab{a}})Zhang, Malkin, Liu, Volokhova,
  Courville, and Bengio]{zhang2022generative}
Zhang, D., Malkin, N., Liu, Z., Volokhova, A., Courville, A., and Bengio, Y.
\newblock Generative flow networks for discrete probabilistic modeling.
\newblock \emph{arXiv preprint arXiv:2202.01361}, 2022{\natexlab{a}}.

\bibitem[Zhang et~al.(2022{\natexlab{b}})Zhang, Liu, and
  Liu]{zhang2022langevin}
Zhang, R., Liu, X., and Liu, Q.
\newblock A {Langevin}-like sampler for discrete distributions.
\newblock In \emph{International Conference on Machine Learning}, pp.\
  26375--26396. PMLR, 2022{\natexlab{b}}.

\end{thebibliography}
\bibliographystyle{icml2022}
}

\newpage 
\appendix

\begin{center}
\LARGE
\textbf{Appendix for ``Training Discrete EBMs with Energy Discrepancy''}
\end{center}

\etocdepthtag.toc{mtappendix}
\etocsettagdepth{mtchapter}{none}
\etocsettagdepth{mtappendix}{subsection}
{\small \tableofcontents}

\section{Abstract Proofs and Derivations}
\subsection{Proof of the Non-Parametric Estimation Theorem \ref{theorem-energy-discrepancy}} \label{appendix-proof-theorem-energy-discrepancy}
In this subsection we give a formal proof for the uniqueness of minima of $\ED_{q}(p_\data, U)$ as a functional in the energy function $U$. We first reiterate the theorem as stated in the paper:
\restatheoremone*
We test energy discrepancy on the first and second order optimality conditions, i.e. we test that the first functional derivative of ED vanishes in $U^\ast$ and that the second functional derivative is positive definite. For uniqueness and well-definedness, we constrain the optimisation domain to the following set:
\begin{equation*}
    \mathcal G := \left\{U:\mathcal X\mapsto \mathbb R \,\text{ such that }\, \exp(-U)\in L^1(\mathcal X, \mathrm d\x)\,,\,\,\, U\in L^1(p_\data)\,, \,\text{ and }\, \min_{\x\in\mathcal X} U(x) = 0\right\}
\end{equation*}
and require that there exists a $U^\ast\in \mathcal G$ such that $\exp(-U^\ast) \propto p_\data$. We now start with the following lemmata and then complete the proof of \cref{theorem-energy-discrepancy} in \cref{corollary-energy-discrepancy}.

{\lemma{
Let $h\in \mathcal G$ be arbitrary. The first variation of $\ED_q$ is given by
\begin{align}\label{first-variation-equation}
    \left. \frac{\mathrm{d}}{\mathrm{d} \epsilon} \ED_q (p_\data, U + \epsilon h) \right|_{\epsilon = 0} = \mathbb{E}_{p_\data(\x)} [h(\x)] - \mathbb{E}_{p_\data(\x)} \mathbb{E}_{q(\y|\x)} \mathbb{E}_{p_U(\z | \y)}[h(\z)]
\end{align}
where $p_U(\z | \y) = \frac{q(\y|\z)\exp(-U(\z))}{\sum_{\z'\in\mathcal X} q(\y\vert \z') \exp(-U(\z'))}$.
}}
\begin{proof}
We define the short-hand notation $U_\epsilon := U + \epsilon h$. The energy discrepancy at $U_\varepsilon$ reads
\begin{align*}
    \ED_{q} (p_\data, U_\epsilon) =  \mathbb{E}_{p_\data(\x)} [U_\epsilon(\x)] + \mathbb{E}_{p_\data(\x)}\mathbb{E}_{q(\y | \x)} \left[ \log \sum_{\z\in\mathcal X} q(\y | \z) \exp(-U_\epsilon(\z)) \right] \nonumber\,.
\end{align*}
For the first functional derivative, we only need to calculate
\begin{align}\label{equation-appendix-derivative-free-energy}
    \frac{\mathrm{d}}{\mathrm{d} \epsilon} \log \sum_{\z\in\mathcal X} q(\y | \z) \exp(-U_\epsilon(\z)) = \sum_{\z\in\mathcal X} \frac{- q(\y | \z) h(\z) \exp(-U_\epsilon(\z))}{\sum_{\z'\in\mathcal X} q(\y | {\z^\prime}) \exp(-U_\epsilon(\z^\prime)) } = -\mathbb{E}_{p_{U_\epsilon}(\z | \y)}[h(\z)].
\end{align}
Plugging this expression into $\ED_{q} (p_\data, U_\epsilon)$ and setting $\epsilon = 0$ yields the first variation of $\ED_q$.
\end{proof}
{\lemma{ \label{second-variation-energy-discrepancy}
The second variation of $\ED_q$ is given by
\begin{align*}
    \left. \frac{\mathrm{d}^2}{\mathrm{d} \epsilon^2} \ED_{q} (p_\data, U + \epsilon h) \right|_{\epsilon = 0} = \mathbb{E}_{p_\data(\x)} \mathbb{E}_{q(\y|\x)} \mathrm{Var}_{p_{U}(\z| \y)}[h(\z)].
\end{align*}
}}
\begin{proof}
For the second order term, we have based on equation \ref{equation-appendix-derivative-free-energy} and the quotient rule for derivatives:
\begin{align*}
    \frac{\mathrm{d}^2}{\mathrm{d} \epsilon^2} \log &\sum_{\z\in\mathcal X} q(\y | \z) \exp(-U_\epsilon(\z))\\\notag
    &= \frac{\sum_{\z\in\mathcal X} q(\y | \z) \exp(U_\epsilon(\z))\,h^2(\z)\, \, \sum_{\z^\prime\in\mathcal X} q(\y | \z^\prime) \exp(-U_\epsilon(\z^\prime)) }{\left( \sum_{\z^\prime\in\mathcal X } q(\y | {\z^\prime}) \exp(-U_\epsilon({\z^\prime})) \right)^2} \\\notag
    &\quad- \frac{\sum_{\z\in\mathcal X} q(\y | \z) \exp(U_\epsilon(\z))h(\z) \sum_{\z^\prime\in\mathcal X} q(\y | {\z^\prime}) \exp(-U_\epsilon({\z^\prime})) h(\z^\prime) }{\left( \sum_{\z^\prime\in\mathcal X} q(\y | {\z^\prime}) \exp(-U_\epsilon({\z^\prime}))  \right)^2} \\\notag
    &= \mathbb{E}_{p_{U_\epsilon}(\z| \y)}[h^2(\z)] - \mathbb{E}_{p_{U_\epsilon}(\z| \y)}[h(\z)]^2 = \mathrm{Var}_{p_{U_{\epsilon}}(\z| \y)}[h(\z)]\,.
\end{align*}
We obtain the desired result by interchanging the outer expectations with the derivatives in $\epsilon$.
\end{proof}
{\corollary{\label{corollary-energy-discrepancy}
Let $c=\min_{\x\in\mathcal X} (-\log p_\data(\x))$. For $U^\ast = -\log(p_\data) - c\in\mathcal G$ it holds that
\begin{align*}
    \left. \frac{\mathrm{d}}{\mathrm{d} \epsilon} \ED_{q} (p_\data, U^\ast + \epsilon h) \right|_{\epsilon = 0} &= 0  \\\nonumber
    \left. \frac{\mathrm{d}^2}{\mathrm{d} \epsilon^2} \ED_{q} (p_\data, U^\ast + \epsilon h) \right|_{\epsilon = 0}&>0 \hspace{1cm}\text{for all} \hspace{1cm}h\,,
\end{align*}
Furthermore, $U^\ast$ is the unique global minimiser of $\ED_q(p_\data, \cdot)$ in $\mathcal G$.
}}

\begin{proof}
By definition, the variance is non-negative, i.e. for every $h\in \mathcal G$:
\begin{equation*}
    \left. \frac{\mathrm{d}^2}{\mathrm{d} \epsilon^2} \ED_{q} (p_\data, U + \epsilon h) \right|_{\epsilon = 0} = \mathrm{Var}_{p_{U}(\z| \y)}[h(\z)]\geq 0\,.
\end{equation*}
Consequently,  the energy discrepancy is convex and an extremal point of $\ED_q(p_\data, \cdot)$ is a global minimiser. We are left to show that the minimiser is obtained at $U^\ast$ and unique. First of all, we have for $U^\ast$:
\begin{align*}
    \mathbb{E}_{p_{U^\ast}(\z | \y)}[h(\z)]
    &= \sum_{\z\in\mathcal X}\frac{q(\y|\z) \exp(-U^*(\z))}{\sum_{\z^\prime\in\mathcal X} q(\y|\z^\prime) \exp(-U^*(\z^\prime)) }  h(\z)\nonumber \\
    &= \sum_{\z\in\mathcal X}\frac{q(\y|\z) p_\data(\z)}{\sum_{\z^\prime \in\mathcal X}q(\y|\z^\prime) p_\data(\z^\prime)}  h(\z). \nonumber
\end{align*}
By applying the outer expectations we obtain
\begin{align*}
    \mathbb{E}_{p_\data(\x)} \mathbb{E}_{q(\y|\x)} \mathbb{E}_{p_{U^\ast}(\z| \y)}[h(\z)] 
    &= \sum_{\x\in\mathcal X} p_\data(\x) \sum_{\y\in\mathcal Y} \left(q(\y\vert \x)   \sum_{z\in\mathcal X}\left(\frac{q(\y|\z) p_\data(\z)}{\sum_{\z'\in\mathcal X} q(\y|\z^\prime) p_\data(\z^\prime)}  h(\z) \right)\right)\,\nonumber \\
    &= \sum_{\z\in\mathcal X} \sum_{\y\in\mathcal Y} q(\y\vert \x) p_\data(\z) h(\z) \nonumber \\
    &= \mathbb{E}_{p_\data(\z)} [h(\z)], \nonumber
\end{align*}
where we used that the marginal distributions $\sum_{\x\in\mathcal X} p_\data(\x) q(\y\vert \x) $ cancel out and the conditional probability density integrates to one. This implies
\begin{align*}
    \left. \frac{\mathrm{d}}{\mathrm{d} \epsilon} \ED_{q} (p_\data, U^* + \epsilon h) \right|_{\epsilon = 0} =\mathbb{E}_{p_\data(\z)} [h(\z)]-\mathbb{E}_{p_\data(\z)} [h(\z)] = 0. \nonumber
\end{align*}
for all $h\in\mathcal G$. We now show that
\begin{equation*}
    \left. \frac{\mathrm{d}^2}{\mathrm{d} \epsilon^2} \ED_{q} (p_\data, U^\ast + \epsilon h) \right|_{\epsilon = 0} = \mathbb{E}_{p_\data(\x)} \mathbb{E}_{q(\y|\x)} \mathrm{Var}_{p_\data(\z| \y)}[h(\z)] >0\,.
\end{equation*}
Assume that the second variation was zero. Since the perturbed data distribution $\sum_{\x\in\mathcal X} p_\data(\x)q(\y\vert \x)$ is positive, the second variation at $U^\ast$ is zero if and only if the conditional variance $\mathrm{Var}_{p_\data(\z\vert \y)}[h(\z)] = 0$. Since $U^\ast+\varepsilon h\in\mathcal G$, the function $h$ can not be constant. By definition of the conditional variance, $h(\z)$ must then be a deterministic function of $\y \sim \sum_{\x\in\mathcal X} q(\y\vert \x)p_\data(\x)$. Since $h$ was arbitrary, there exists a measurable map $g$ such that $\z = g(\y)$ and $\mathrm{Var}_{p_\data(\z\vert\y)}[\z] = 0 $ which is a contradiction to our assumptions. Consequently, $U^\ast$ is the unique global minimiser of $\ED_q$ which completes the statement in \cref{theorem-energy-discrepancy}.
\end{proof}
\section{Connections to other Methods}
In this section, we follow \citet{schroeder2023energy}.
\subsection{Connections of Energy Discrepancy with Contrastive Divergence}\label{appendix-subsection-connection-CD}
The contrastive divergence update can be derived from an energy discrepancy when, for $E_\theta$ fixed, $q$ satisfies the detailed balance relation
\begin{equation*}
    q(\y\vert\x)\exp(-E_\theta(\x)) = q(\x\vert \y)\exp(-E_\theta(\y))\,.
\end{equation*}
To see this, we calculate the contrastive potential induced by $q$: We have
\begin{align*}
    -\log \sum_{\x'\in\mathcal X} q(\y\vert \x')\exp(-E_\theta(\x')) &= -\log \sum_{\x'\in\mathcal X} q(\x'\vert \y)\exp(-E_\theta(\y))= E_\theta(\y)\,.
\end{align*}
Consequently, the energy discrepancy induced by $q$ is given by
\begin{equation*}
    \ED_q(p_\data, E_\theta) = \mathbb E_{p_\data(\x)}[E_\theta(\x)]-\mathbb E_{p_\data(\x)}\mathbb E_{q(\y\vert\x)}[E_\theta(\y)]\,.
\end{equation*}
Updating $\theta$ based on a sample approximation of this loss leads to the contrastive divergence update
\begin{equation*}
   \Delta\theta \propto \frac{1}{N}\sum_{i=1}^N \nabla_\theta E_\theta(\x^i)- \frac{1}{N}\sum_{i=1}^N \nabla_\theta E_\theta(\y^i) \quad \y^i \sim q(\cdot\vert \x^i)
\end{equation*}
It is important to notice that the distribution $q$ depends on $E_\theta$ and needs to adjusted in each step of the algorithm. For fixed $q$, $\ED_q(p_\data, E_\theta)$ satisfies \cref{theorem-energy-discrepancy}. This means that each step of contrastive divergence optimises a loss with minimiser $E_\theta^\ast = -\log p_\data +c $. However, the loss function changes in each step of contrastive divergence. The connection also highlights the importance Metropolis-Hastings adjustment to ensure that the implied $q$ distribution satisfies the detailed balance relation.
\subsection{Derivation of Energy Discrepancy from KL Contractions}\label{appendix-subsection-KLcontraction}
A Kullback-Leibler contraction is the divergence function $\infdiv{p_\data}{p_\ebm} - \infdiv{Qp_\data}{Qp_\ebm}$ \citep{Lyu2011KLContractions} for the convolution operator $Q p(\y) = \sum_{\x'\in\mathcal X} q(\y\vert \x')p(\x')$. The linearity of the convolution operator retains the normalisation of the measure, i.e. for the energy-based distribution $p_\ebm$ we have
\begin{equation*}
    Q p_\ebm = \frac{1}{Z_U}\sum_{\x'\in\mathcal X} q(\y\vert \x') \exp(-U(\x')) \quad \text{with} \quad Z_U = \sum_{\x'\in\mathcal X} \exp(-U(\x'))\,.
\end{equation*}
The KL divergences then become with $U_q := -\log Q\exp(-U(\x))$
\begin{align*}
    \infdiv{p_\data}{p_\ebm} &= \mathbb E_{p_\data(\x)}[\log p_\data(\x)] + \mathbb E_{p_\data(\x)}[U(\x)] + \log Z_U\\\notag
    \infdiv{Qp_\data}{Qp_\ebm}&= \mathbb E_{Qp_\data(\y)}[\log Qp_\data(\y)] + \mathbb E_{Qp_\data(\y)}\left[U_q(\y)\right] + \log Z_U
\end{align*}
Since the normalisation cancels when subtracting the two terms we find
\begin{equation*}
    \infdiv{p_\data}{p_\ebm} - \infdiv{Qp_\data}{Qp_\ebm} = \ED_q(p_\data, U) + c
\end{equation*}
where $c$ is a constant that contains the $U$-independent entropies of $p_\data$ and $Qp_\data$.
\section{Sample Approximations of Energy Discrepancies}
In this section, we discuss practical implementations of the mean-pooling transform as an information destroying deterministic process and the grid-neighbourhood as a neighbourhood-based transformation.
\subsection{General Strategy}
As a general strategy, the contrastive potential has to be written as an expectation over an appropriate to be determined distribution $p_{\mathrm{neg}, q, \y}$ that depends on the chosen perturbation process and on the point where the contrastive potential is evaluated, i.e.
\begin{equation}
    U_q(\y) = -\log \mathbb E_{p_{\mathrm{neg}, q, \y}(\x')} \exp(-U(\x'))
\end{equation}
which allows the evaluation of the contrastive potential via sampling from $p_{\mathrm{neg}, q, \y}$. The energy discrepancy can then be written as
\begin{equation}
    \ED_q(p_\data, U) = \mathbb E_{p_\data(\x)}\mathbb E_{q(\y\vert\x)}\left[\log \mathbb E_{p_{\mathrm{neg}, q, \y}(\x')} \left[\exp(U(\x)-U(\x'))\right]\right]
\end{equation}
by using properties of the logarithm and exponential and the fact that $U(\x)$ does not depend on the expectations taken in $\y$ and $\x'$. The loss can then be approximated via ancestral sampling. We first sample a batch $\x_+^i\sim p_\data$, subsequently sample its perturbed counter part $\y^i\sim q(\cdot\vert\x_+^i)$, and finally sample $M$ negative samples $\x_-^{i, j}\sim p_{\mathrm{neg}, q, \y^i}$. Sometimes, the perturbed sample $\y^i$ is never explicitely computed in the process. As described in \cref{equ:stabilised-loss-function}, the approximation is always stabilised through tunable hyper-parameter $w$ which finally yields the loss function
\begin{equation*}
    \mathcal L_{q, M, w}(U) := \frac{1}{N} \sum_{i=1}^N \log\left(w+ \sum_{j= 1}^M \exp(U(\x^i_+) - U(\x_-^{i,j}))\right) - \log(M)
\end{equation*}
The justification for the stabilisation is two-fold. Firstly, the logarithm makes the Monte-Carlo approximation of the contrastive potential biased due to Jensens inequality. The bias is negative, given to leading order by the variance of the approximation, and depends on the energy function $U$. Thus, the optimiser may start to optimise for a high bias and high variance estimator of the contrastive potential rather than learning the data distribution. While this issue can be alleviated by significantly large choices for $M$, it is much more practical to introduce a deterministic lower bound to the loss-functional through the stabilisation $w$, which prevents the bias and logarithm from diverging. Secondly, the effect of the stabilisation goes to zero as $M$ increases. Thus, the asymptotic limit for $M$ and $N$ large is retained through the stabilisation. For more details and analogous arguments in the continuous case, see \citet{schroeder2023energy}.
\subsection{Mean Pooling Transform}\label{subsct:mean-pooling-appendix}
We describe the mean-pooling transform on the example of image data which takes values in the space $\{0, 1\}^{h\times w}$. We fix a window size $s$ and reshape each data-point into blocks of size $s\times s$, i.e.
\begin{equation*}
    \{0, 1\}^{h\times w} \to \{0, 1\}^{s\times s \times \frac{h}{s} \times \frac{w}{s}}\,, \quad \x\mapsto \bar \x
\end{equation*}
The mean pooling transform $g_{\mathrm{pool}}$ computes the average over each block $\bar \x_{\bullet, \bullet, i, j}$ for $i=1, 2, \dots, h/s$ and $j=1, 2, \dots, w/s$. The corresponding preimage of the mean pooling transform is given by the set of points which are identical to $\x$ up to block-wise permutation, i.e.
\begin{equation*}
    g^{-1}(g_\mathrm{pool}(\x))=\{ \x'\in \mathcal X: \text{ there exist } \pi_{i, j}\in S_{s\times s} \text{ s.t. } \bar\x'_{l, k, i, j} = \bar\x'_{\pi_{i, j}(l, k), i, j} \text{ for all } l, k, i, j\}
\end{equation*}
where $S_{s\times s}$ denotes the permutation group for matrices of size $s\times s$. In practice, the mean-pooled data point has to never be computed, only the block wise permutations of the data point are required. Consequently, we obtain negative samples through $\x_-^{i, j} \sim \mathcal U(g^{-1}(g_\mathrm{pool}(\x^i)))$, i.e. via block wise permutation of the entries of each data point $\x^i$.

Strictly speaking, this transformation violates the assumptions of \cref{theorem-energy-discrepancy} for data points that only consist of blocks that average to $1$ or $0$. Since this is only the case for a small set of the state space, we assume this violation to be negligible. 
\subsection{Grid Neighborhood}\label{subsct:grid-neighbourhood-appendix}
The grid neighbourhood for $\x\in \{0, 1\}^d$ is constructed as
\begin{equation*}
    \mathcal N_{\mathrm{grid}}(\x) = \{ \y\in \{0, 1\}^d\, :\, \y-\x = \pm \e_k,\, k = 1, 2, \dots, d\}
\end{equation*}
where $\e_k$ is a vector of zeros with a one in the k-th entry. This neighbourhood structure is symmetric, i.e. $\mathcal N_{\mathrm{grid}}^{-1}(\y) = \mathcal N_{\mathrm{grid}}(\y)$. Consequently, the negative samples are created by sampling from
\begin{equation*}
    \x_-^{i, j} \sim \mathcal U(\mathcal N_{\mathrm{grid}}( \y^i)) \quad \text{ with }\quad \y^i\sim  \mathcal U(\mathcal N_{\mathrm{grid}}( \x^i))
\end{equation*}
Notice that each negative sample is the second neighbour of the positive sample, and with a small chance the positive sample itself.

\subsection{Directed Neighbourhood Structures} \label{appendix-sec-directed-neighbour-structures}
More generally, the neighbourhood structure may form a non-symmetric directed graph for which the neighbourhood maps $\mathcal N^{-1}$ and $\mathcal N$ don't coincide. In this case, an additional weighting-term is introduced. We denote the number of neighbours of $\x$ as $K_\x = \vert \mathcal N(\x)\vert $ and the number of elements of which $\y$ is a neighbour as $K'_\y = \vert \mathcal N^{-1}(\y)\vert$. The forward transition density is given by the uniform distribution, i.e.
\begin{equation}
    q(\y\vert\x) = \left\{\begin{matrix}
        1/K_\x & \text{if} &\y\in\mathcal N(\x) \\
        0 & \text{else} &
    \end{matrix}\right.
\end{equation}

We then have
\begin{align*}
    U_{\mathcal N}(\y) &= \log \sum_{\x'\in\mathcal X} q(\y\vert\x') \exp(-U(\x')) \\\notag
    &= \log \sum_{\x'\in \mathcal N^{-1}(\y)} \frac{1}{K_{\x'}}\exp(-U(\x')) \\\notag
    &= \log \frac{1}{K'_\y}\sum_{\x'\in \mathcal N^{-1}(\y)} \frac{K'_\y}{K_{\x'}}\exp(-U(\x')) \\\notag
    &= \log \mathbb E_{\x' \sim \mathcal U(\{\mathcal N^{-1}(\y)\})}[\omega_{\y\x'}\exp(-U(\x'))]
\end{align*}
where we introduced the weighting term $\omega_{\y\x'} = K'_\y / K_{\x'}$.
\subsection{Consistency of our Approximation}\label{subsct:consistency-proof-appendix}
The following proof is similar to \citet{schroeder2023energy}. We first restate the consistency result:
\restatetheoremtwo*
\begin{proof}
    For $N$ data points $\x_+^i\sim p_\data$ and perturbed points $\y^i\sim q(\cdot \vert\x_+^i)$ denote the $M$ corresponding negative samples by $\x_-^{i, j}\sim p_{\mathrm{neg}, q, \y^i}$. Notice that the distribution of the negative samples depends on $\y^i$. Using the triangle inequality, we can upper bound the difference $\lvert \ED_{q}(p_\data, U)-\mathcal L_{q, M, w}(U)\rvert$ by upper bounding the following two terms, individually:
    \begin{align*}
          \Bigg\lvert \ED_{q}(p_\data, U) - & \frac{1}{N}\sum_{i=1}^N \log \mathbb E\left[\exp(U(\x_+^i)-U(\x_-^{i, j})\,\Big\vert\, \x_+^i, \y^i\right]\Bigg\rvert\\\notag
        + &\left\lvert \frac{1}{N}\sum_{i=1}^N \log \mathbb E\left[\exp(U(\x_+^i)-U(\x_-^{i, j})\,\Big\vert\, \x_+^i, \y^i\right] - \mathcal L_{q, M, w}(U)\right\rvert
    \end{align*}
The conditioning expresses that the expectation is only taken in $\x_-^{i, j}\sim p_{\mathrm{neg}, q, \y^i}$ while keeping the values of the random variables $\x_+^i$ and $\y^i$ fixed. The first term can be bounded by a sequence $\varepsilon_N\xrightarrow{a.s.} 0$ due to the normal strong law of large numbers. For the second term one needs to consider that the distribution $p_{\mathrm{neg}, q, \y^i}$ depends on the random variable $\y^i$. For this reason, we notice that $\x_-^{i, j}$ are conditionally indepedent given $\x_+^i, \y^i$ and employ a conditional version of the strong law of large numbers \citep[Theorem 4.2]{ConditionalLLN} to obtain
  \begin{equation*}
    \frac{1}{M}\sum_{j=1}^M \exp\left(U(\x_+^i) - U(\x_-^{i, j})\right)\xrightarrow{a.s.} \mathbb E\left[\exp(U(\x_+^i)-U(\x_-^{i, j})\,\Big\vert\, \x_+^i, \y^i\right]
  \end{equation*}
Next, we have that the deterministic sequence $w/M\to 0$. Thus, adding the stabilisation $w/M$ does not change the limit in $M$. Furthermore, since the logarithm is continuous, the limit also holds after applying the logarithm. Finally, the estimate translates to the sum by another application of the triangle inequality:
For each $i = 1, 2, \dots, N$ there exists a sequence $\varepsilon_{i, M}\xrightarrow{a.s.} 0$ such that
\begin{align*}
    &\left\lvert \frac{1}{N}\sum_{i=1}^N \log \mathbb E\left[\exp(U(\x_+^i)-U(\x_-^{i, j})\,\Big\vert\, \x_+^i, \y^i\right] - \mathcal L_{q, M, w}(U)\right\rvert \\\notag
    & \leq \frac{1}{N}\sum_{i=1}^N \left\lvert \log \mathbb E\left[\exp(U(\x_+^i)-U(\x_-^{i, j})\,\Big\vert\, \x_+^i, \y^i\right] - \log\frac{1}{M}\sum_{j=1}^M \exp\left(U(\x_+^i) - U(\x_-^{i, j})\right) \right\rvert \\\notag
    & <  \frac{1}{N}\sum_{i=1}^N \varepsilon_{i, M} \leq \max(\varepsilon_{1, M}, \dots, \varepsilon_{N, M})\,.
\end{align*}
Hence, for each $\varepsilon>0$ there exists an $N\in\mathbb N$ and an $M(N)\in \mathbb N$ such that $\lvert \ED_{q}(p_\data, U)-\mathcal L_{q, M(N), w}(U)\rvert < \varepsilon$ almost surely.
\end{proof}
\section{Related Work}
\paragraph{Contrastive loss functions}
Our work is based on an unpublished work on energy discrepancies in the continuous case \citep{schroeder2023energy}. The motivation for such constructed loss functions lies in the data processing inequality. A similar loss has been suggested before as KL contraction divergence \citep{Lyu2011KLContractions}, however, only for its theoretical properties. Interestingly, the structure of the stabilised energy discrepancy loss shares similarities with other contrastive losses such as \citet{ConditionalNoiseContrastiveEstimation,gutmann2010noise,Oord2018RepresentationLearning}. This poses the question of possible classification-based interpretations of energy discrepancy and of the $w$-stabilisation.

\paragraph{Contrastive divergence and Sampling.}
Discrete training methods for energy-based models largely rely on contrastive divergence methods, thus motivating a lot of work on discrete sampling and proposal methods. Improvements of the standard Gibbs method were proposed by \citet{zanella2020informed} through locally informed proposals. The method was extended to include gradient information \citep{grathwohl2021oops} to drastically reduce the computational complexity of flipping bits of binary valued data and to flipping bits in several places \citep{sun2022path, emami2023plug, sun2022optimalscaling}. Finally, discrete versions of Langevin sampling have been introduced based on this idea \citep{zhang2022langevin,rhodes2022enhanced,sun2023discrete}. Consequently, most current implementations of contrastive divergence use multiple steps of a gradient based discrete sampler. Alternatively, energy-based models can be trained using generative flow networks which learns a Markov chain to construct data by optimising a given reward function. The Markov chain can be used to obtain samples for contrastive divergence without MCMC from the EBM \citep{zhang2022generative}.

\paragraph{Other training methods for discrete EBMs.}
There also exist some MCMC free approaches for training discrete EBMs.
Our work is most similar to concrete score matching \citep{Meng2022concreteSM} which uses neighbourhood structures to define a replacement of the continuous score function. Another sampling free approach for training discrete EBMs is ratio matching \citep{hyvarinen2007some,lyu2012interpretation}. However is has been found that also for ratio matching, gradient information drastically improves the performance \citep{liu2023RMwGGIS}. Moreover, \cite{dai2020learning} proposed to apply variational approaches to train discrete EBMs instead of MCMC. \cite{eikemaapproximate} replaced the widely-used Gibbs algorithms with quasi-rejection sampling to trade off the efficiency and accuracy of the sampling procedure. The perturb-and-map \citep{papandreou2011perturb} is also recently utilised to sample and learn in discrete EBMs \citep{lazaro2021perturb}.

\section{More about Experiments}

\begin{wraptable}{r}{0.6\linewidth}
\vspace{-3mm}
    \small
    \centering
    \caption{Mean negative log-RMSE (higher is better) between the learned connectivity matrix $J_\phi$ and the true matrix $J$ for different values of $D$ and $\sigma$. The results of baselines are directly taken from \cite{zhang2022generative}.} 
    \label{tab:ising_results}
    \resizebox{\linewidth}{!}{
   \begin{tabular}{lccccccc}
        \toprule
         & \multicolumn{5}{c}{$D=10^2$} & \multicolumn{2}{c}{$D=9^2$} \\
        \cmidrule(lr){2-6}\cmidrule(lr){7-8}
         Method $\backslash$ $\sigma$ & $0.1$ & $0.2$ & $0.3$ & $0.4$ & $0.5$ & $-0.1$ & $-0.2$  \\ 
         \midrule
        Gibbs & $4.8$ & $4.7$ & $\bf 3.4$ & $\bf 2.6$ & $\bf 2.3$ & $4.8$ & $4.7$ \\
        GWG & $4.8$ & $4.7$ & $\bf 3.4$ & $\bf 2.6$ & $\bf 2.3$ & $4.8$ & $4.7$ \\
        EB-GFN & $\bf 6.1$ & $\bf 5.1$ & $3.3$ & $\bf 2.6$ & $\bf 2.3$  & $\bf 5.7$ & $\bf 5.1$ \\
        ED-Bern (ours) & $5.1$ & $4.0$ & $2.9$ & $2.5$ &  $\bf 2.3$ & $5.1$ & $4.3$ \\
        ED-Pool (ours) & $4.9$ & $3.6$ & $3.2$ & $\bf 2.6$ & $\bf 2.3$ & $4.9$ & $3.6$ \\
        ED-Grid (ours) & $4.6$ & $4.0$ & $3.1$ & $\bf 2.6$ & $\bf 2.3$ & $4.5$ & $4.0$ \\
        \bottomrule
    \end{tabular}
    }
\vspace{-2mm}
\end{wraptable}

\subsection{Training Ising Models} \label{appendix-sec-training-ising-models}

\textbf{Experimental Details.} As in \cite{grathwohl2021oops,zhang2022generative,zhang2022langevin}, we train a learnable connectivity matrix $J_\phi$ to estimate the true matrix $J$ in the Ising model. To generate the training data, we simulate Gibbs sampling with $1,000,000$ steps for each instance to construct a dataset of $2,000$ samples. For energy discrepancy, we choose $w=1,M=32$ for all variants, $\epsilon=0.1$ in ED-Bern, and the window side is $\sqrt{D} \times \sqrt{D}$ in ED-Pool. The parameter $J_\phi$ is learned by the Adam \citep{kingma2014adam} optimizer with a learning rate of $0.0001$ and a batch size of $256$. Following \cite{zhang2022generative}, all models are trained with an $l_1$ regularization with a coefficient in $\{10, 5, 1, 0.1, 0.01\}$ to encourage sparsity. The other setting is basically the same as Section F.2 in \cite{grathwohl2021oops}. We report the best result for each setting using the same hyperparameter searching protocol for all methods.

\textbf{Quantitative Results.}
We consider $D=10\times 10$ grids with $\sigma = 0.1, 0.2, \dots, 0.5$ and $D=9\times 9$ grids with $\sigma=-0.1, -0.2$. The methods are evaluated by computing the negative log-RMSE between the estimated $J_\phi$ and the ture matrix $J$. As shown in \cref{tab:ising_results}, our methods demonstrate comparable results to the baselines and, in certain settings, even outperform Gibbs and GWG, indicating that energy discrepancy is able to discover the underlying structure within the data.

\begin{figure}[!t]
\centering
\begin{minipage}{\textwidth}
    \begin{minipage}[t]{\samplempwid\textwidth}
    \centering
    \small{2spirals}\\
    \includegraphics[width=\samplefigwid,
    trim=20 10 20 10,clip
    ]{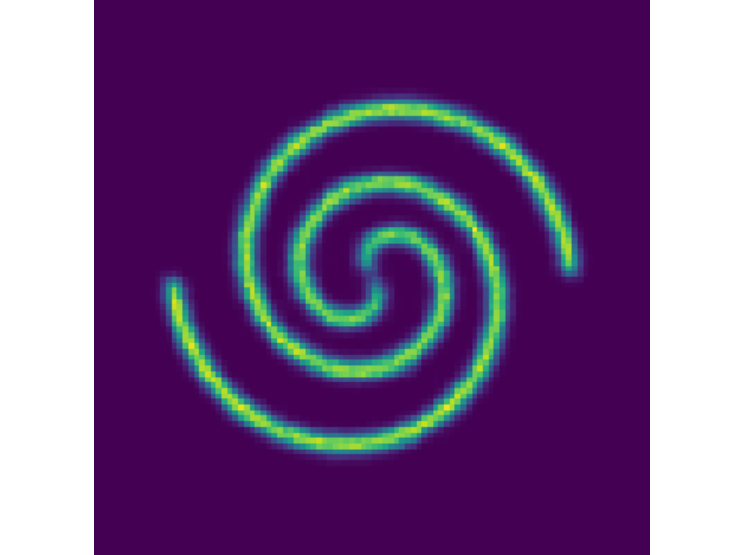}
    \end{minipage}
\hspace{-\samplehinterval}
    \begin{minipage}[t]{\samplempwid\textwidth}
    \centering
    \small{8gaussians}\\
    \includegraphics[width=\samplefigwid
    ,trim=20 10 20 10,clip
    ]{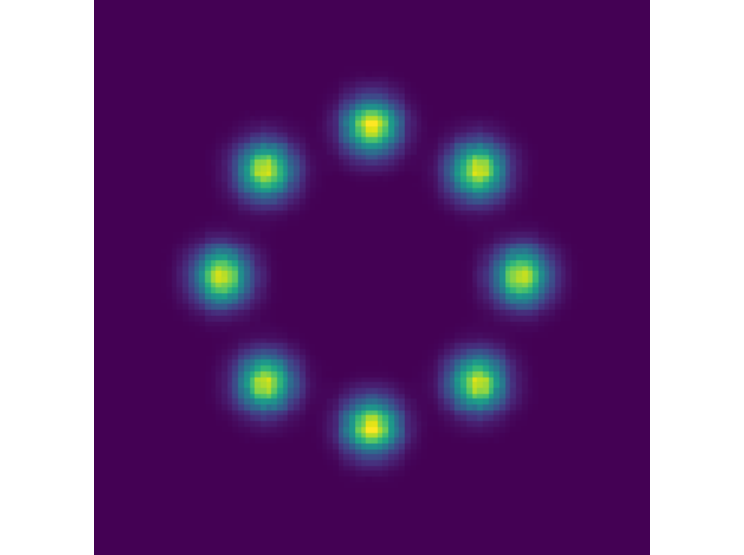}
    \end{minipage}
\hspace{-\samplehinterval}
    \begin{minipage}[t]{\samplempwid\textwidth}
    \centering
    \small{circles}\\
    \includegraphics[width=\samplefigwid
    ,trim=20 10 20 10,clip
    ]{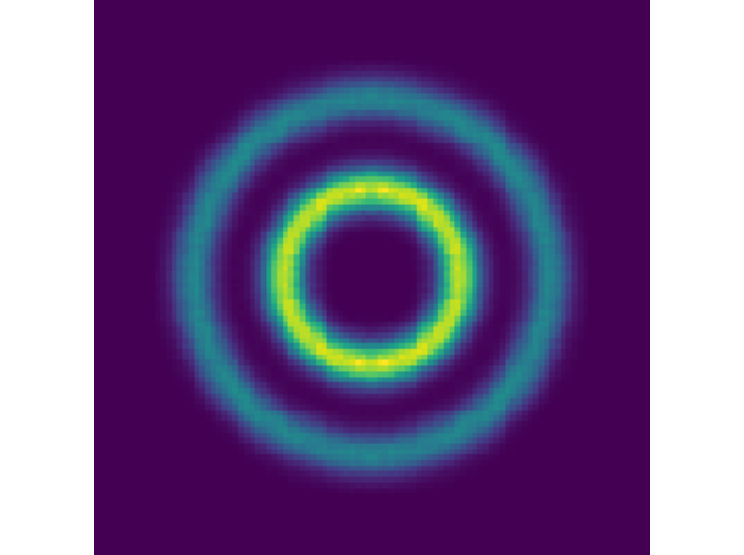}
    \end{minipage}
\hspace{-\samplehinterval}
    \begin{minipage}[t]{\samplempwid\textwidth}
    \centering
    \small{moons}\\
    \includegraphics[width=\samplefigwid
    ,trim=20 10 20 10,clip
    ]{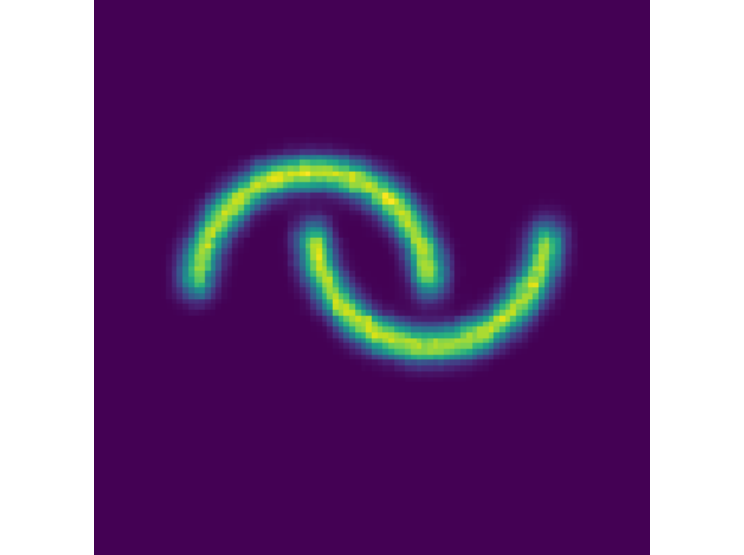}
    \end{minipage}
\hspace{-\samplehinterval}
    \begin{minipage}[t]{\samplempwid\textwidth}
    \centering
    \small{pinwheel}\\
    \includegraphics[width=\samplefigwid
    ,trim=20 10 20 10,clip
    ]{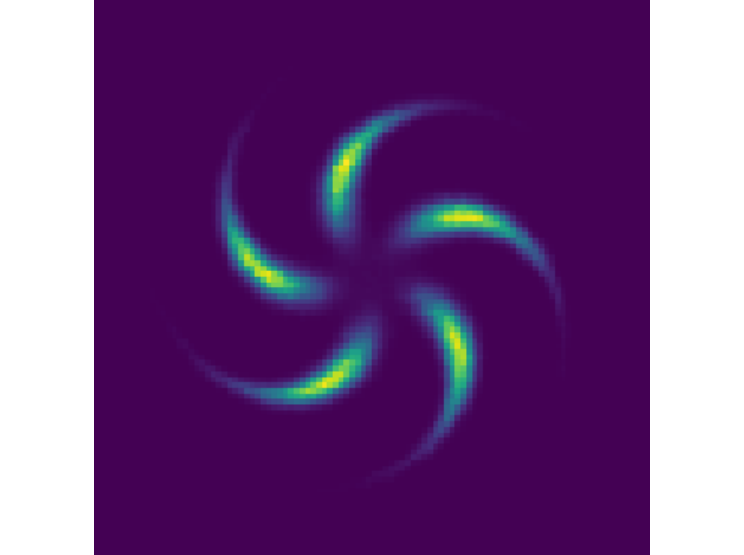}
    \end{minipage}
\hspace{-\samplehinterval}
    \begin{minipage}[t]{\samplempwid\textwidth}
    \centering
    \small{swissroll}\\
    \includegraphics[width=\samplefigwid
    ,trim=20 10 20 10,clip
    ]{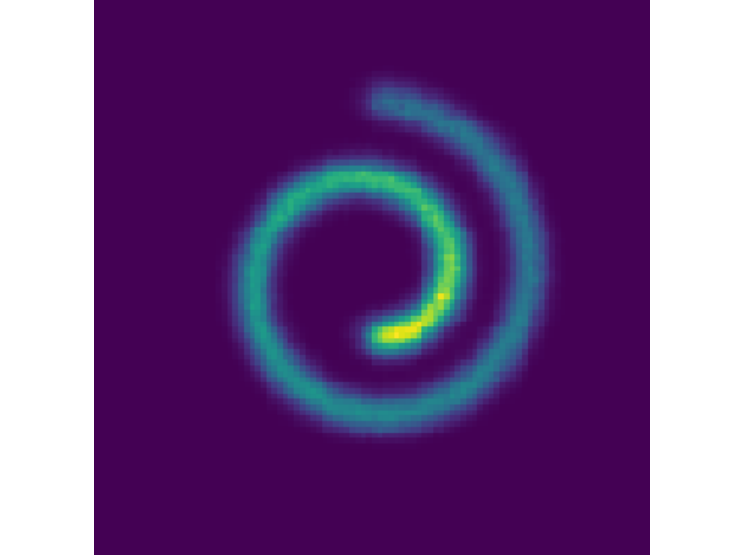}
    \end{minipage}
\hspace{-\samplehinterval}
    \begin{minipage}[t]{\samplempwid\textwidth}
    \centering
    \small{checkerboard}\\
    \includegraphics[width=\samplefigwid
    ,trim=20 10 20 10,clip
    ]{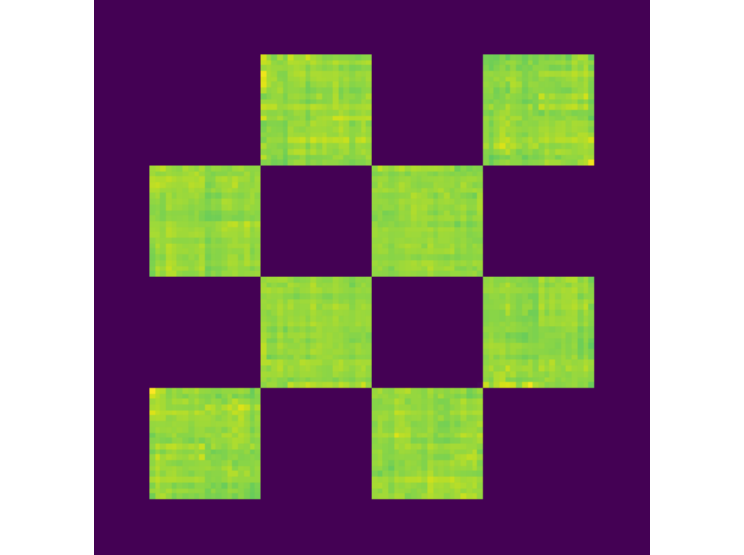}
    \end{minipage}
\hspace{-\samplehinterval}
\end{minipage}
\begin{minipage}{\textwidth}
    \begin{minipage}{\samplempwid\textwidth}
    \centering
    \includegraphics[width=\samplefigwid,trim=20 10 20 10,clip,clip
    ]{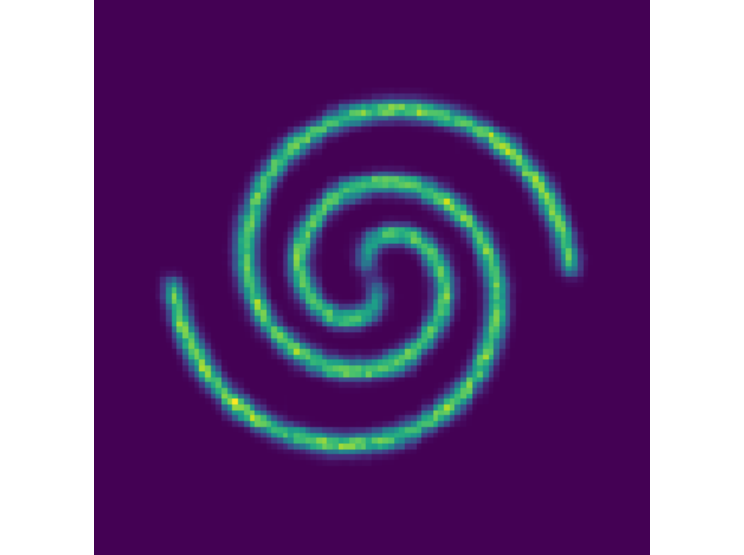}
    \end{minipage}
\hspace{-\hinterval}
    \begin{minipage}{\samplempwid\textwidth}
    \centering
    \includegraphics[width=\samplefigwid,trim=20 10 20 10,clip
    ]{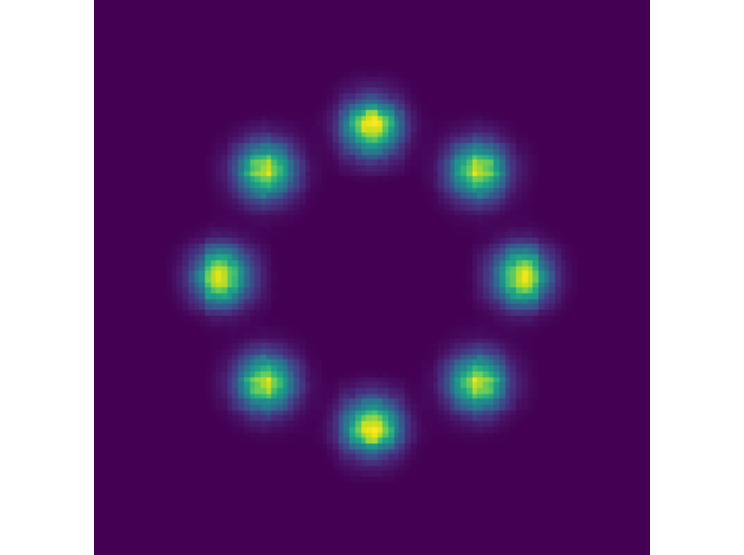}
    \end{minipage}
\hspace{-\hinterval}
    \begin{minipage}{\samplempwid\textwidth}
    \centering
    \includegraphics[width=\textwidth,trim=20 10 20 10,clip
    ]{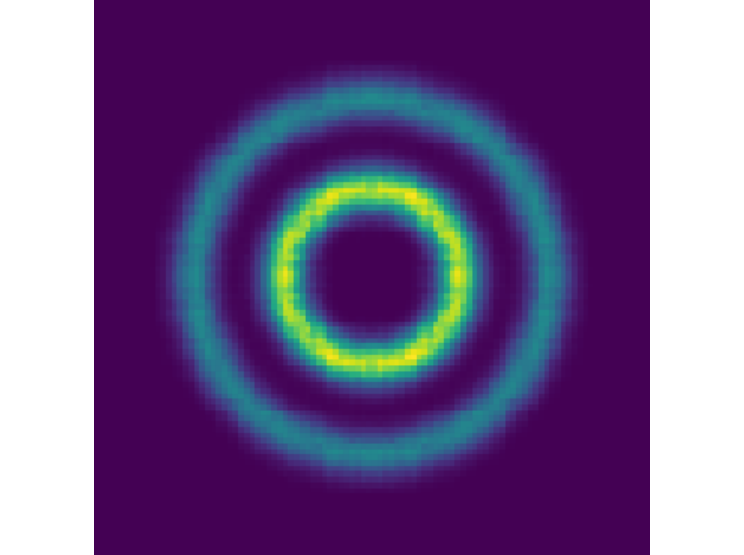}
    \end{minipage}
\hspace{-\hinterval}
    \begin{minipage}{\samplempwid\textwidth}
    \centering
    \includegraphics[width=\textwidth,trim=20 10 20 10,clip
    ]{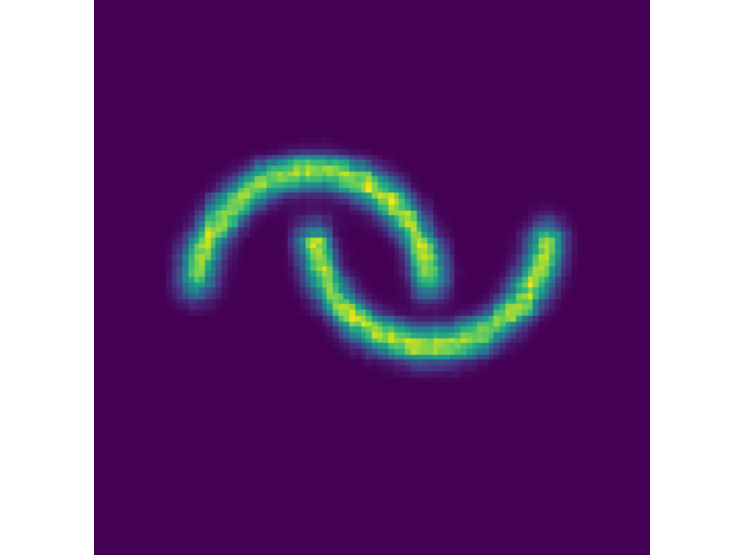}
    \end{minipage}
\hspace{-\hinterval}
    \begin{minipage}{\samplempwid\textwidth}
    \centering
    \includegraphics[width=\textwidth,trim=20 10 20 10,clip
    ]{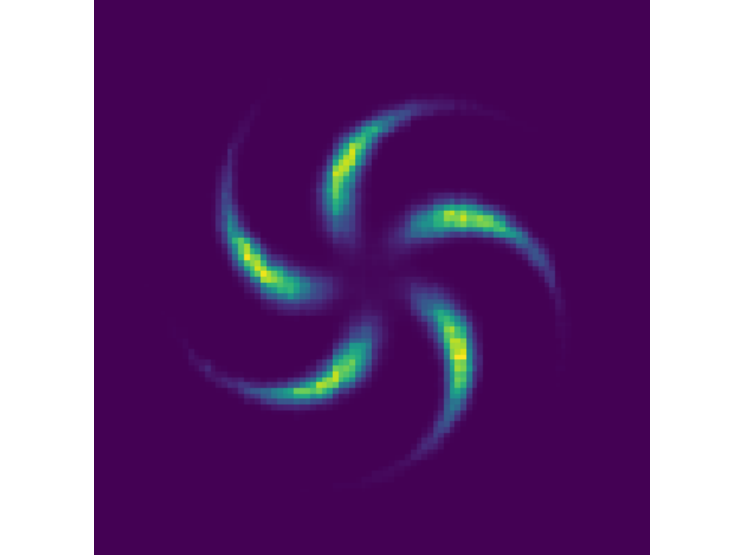}
    \end{minipage}
\hspace{-\hinterval}
    \begin{minipage}{\samplempwid\textwidth}
    \centering
    \includegraphics[width=\textwidth,trim=20 10 20 10,clip
    ]{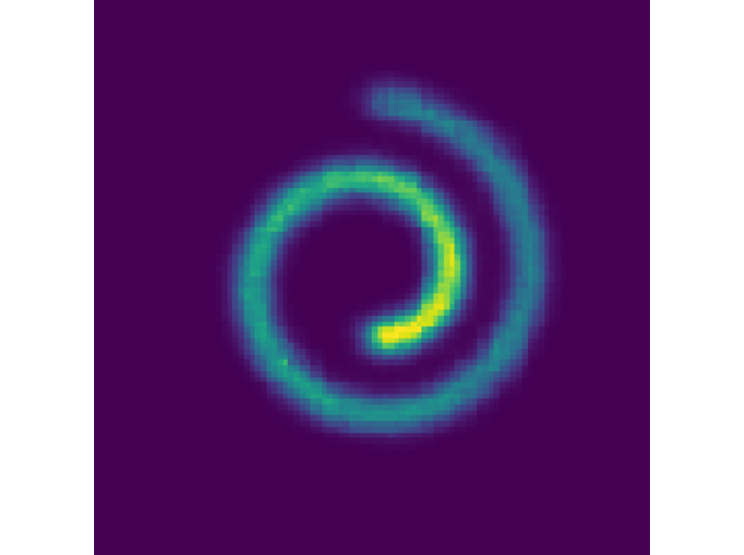}
    \end{minipage}
\hspace{-\hinterval}
    \begin{minipage}{\samplempwid\textwidth}
    \centering
    \includegraphics[width=\textwidth,trim=20 10 20 10,clip
    ]{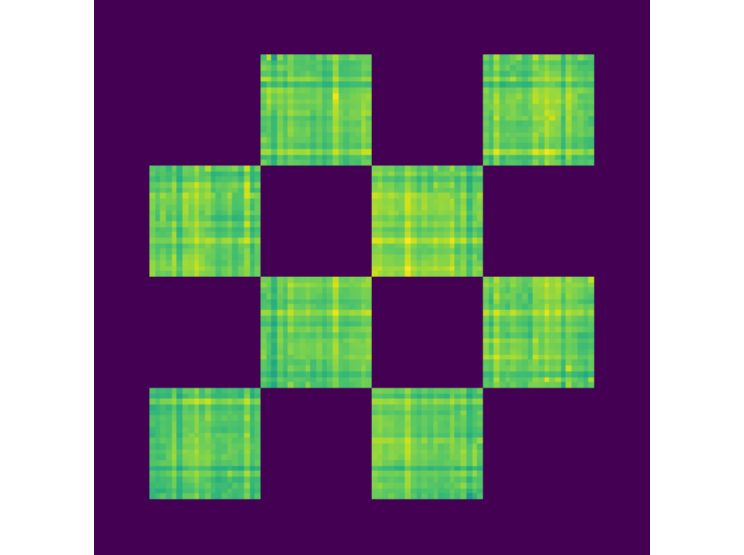}
    \end{minipage}
\hspace{-\hinterval}
\end{minipage}
\begin{minipage}{\textwidth}
    \begin{minipage}{\samplempwid\textwidth}
    \centering
    \includegraphics[width=\samplefigwid,trim=20 10 20 10,clip,clip
    ]{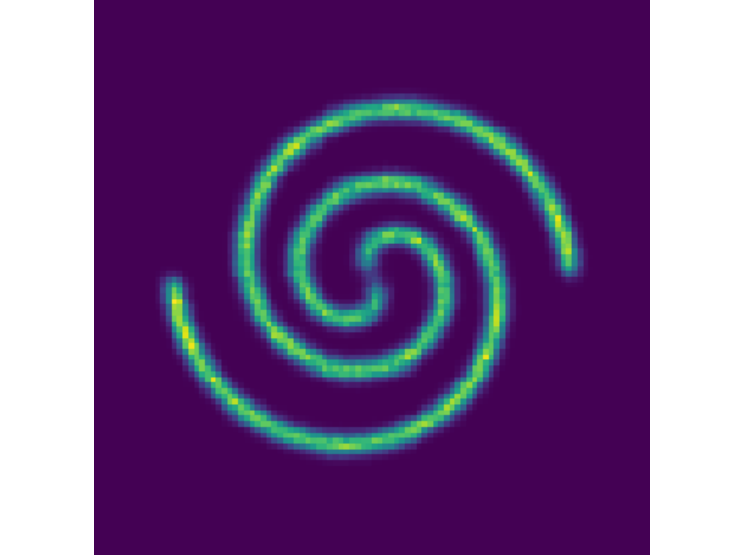}
    \end{minipage}
\hspace{-\hinterval}
    \begin{minipage}{\samplempwid\textwidth}
    \centering
    \includegraphics[width=\samplefigwid,trim=20 10 20 10,clip
    ]{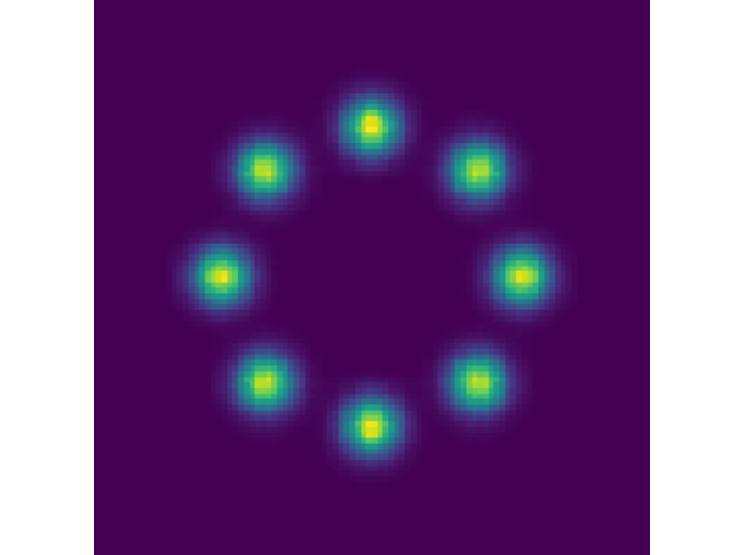}
    \end{minipage}
\hspace{-\hinterval}
    \begin{minipage}{\samplempwid\textwidth}
    \centering
    \includegraphics[width=\textwidth,trim=20 10 20 10,clip
    ]{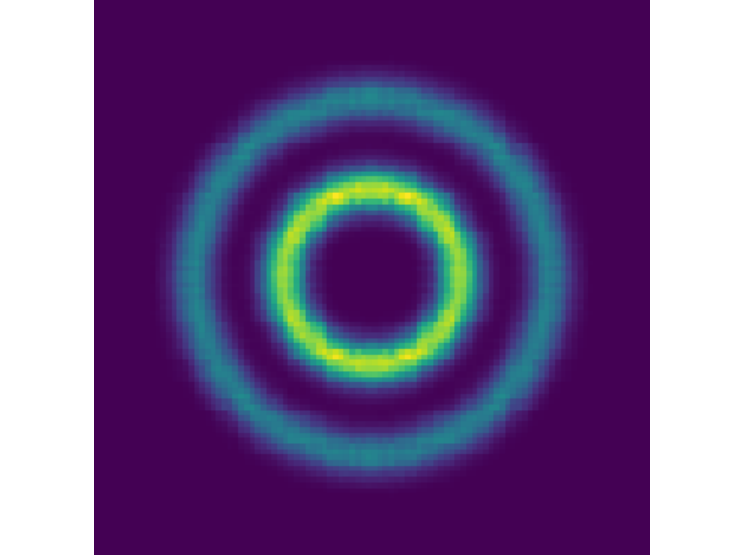}
    \end{minipage}
\hspace{-\hinterval}
    \begin{minipage}{\samplempwid\textwidth}
    \centering
    \includegraphics[width=\textwidth,trim=20 10 20 10,clip
    ]{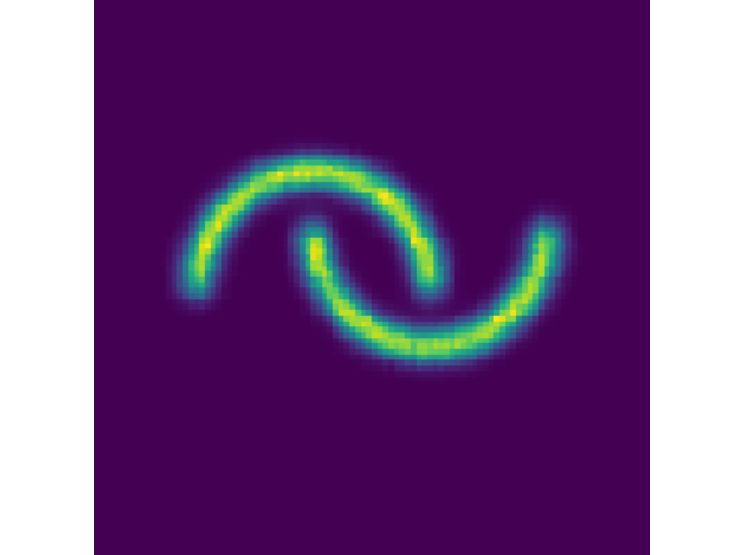}
    \end{minipage}
\hspace{-\hinterval}
    \begin{minipage}{\samplempwid\textwidth}
    \centering
    \includegraphics[width=\textwidth,trim=20 10 20 10,clip
    ]{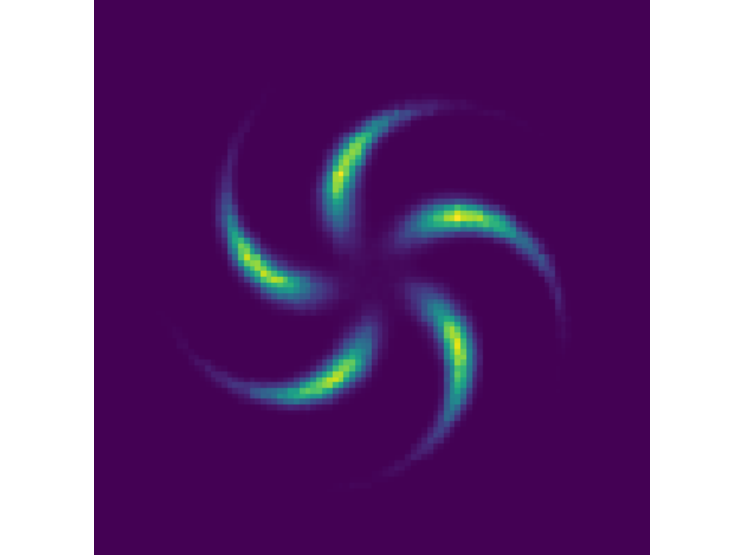}
    \end{minipage}
\hspace{-\hinterval}
    \begin{minipage}{\samplempwid\textwidth}
    \centering
    \includegraphics[width=\textwidth,trim=20 10 20 10,clip
    ]{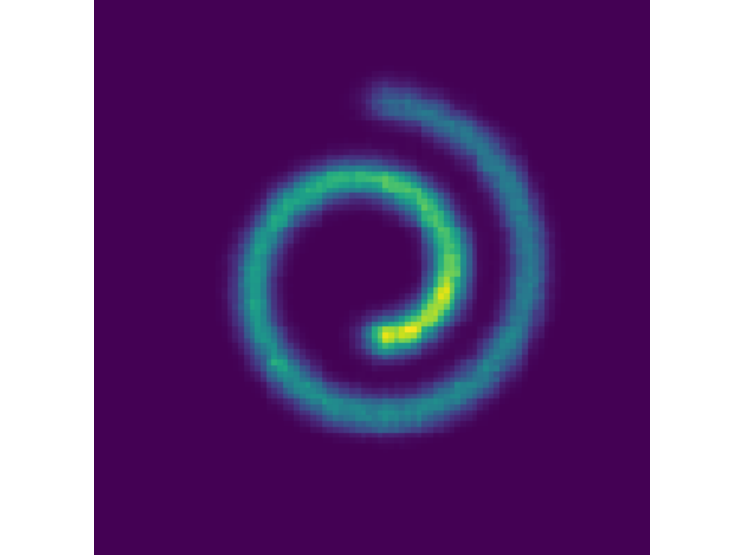}
    \end{minipage}
\hspace{-\hinterval}
    \begin{minipage}{\samplempwid\textwidth}
    \centering
    \includegraphics[width=\textwidth,trim=20 10 20 10,clip
    ]{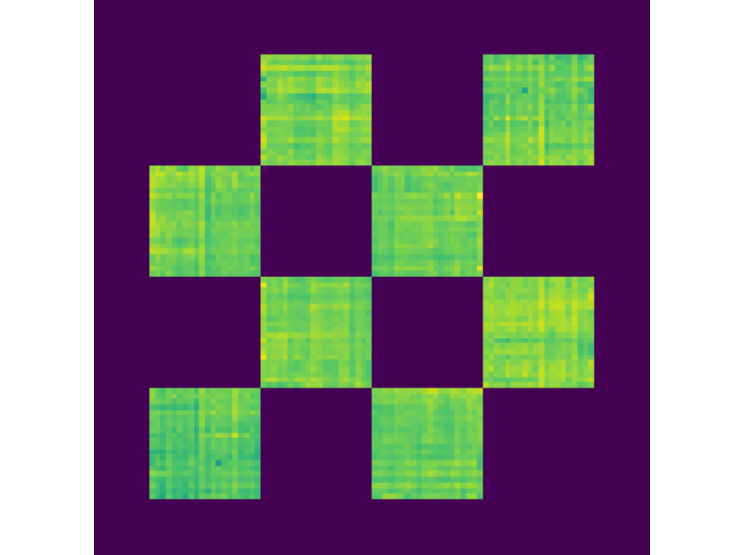}
    \end{minipage}
\hspace{-\hinterval}
\end{minipage}
\centering
\vspace{-2mm}
\caption{
Visualization of the energy function. Top to bottom: ED-Bern, ED-Pool, ED-Grid.
}
\vspace{-2mm}
\label{fig:toy_result_visualisation}
\end{figure}

\begin{figure}[!t]
\begin{minipage}{\textwidth}
    \begin{minipage}{\samplempwid\textwidth}
    \centering
    2spirals
    \includegraphics[width=\textwidth]{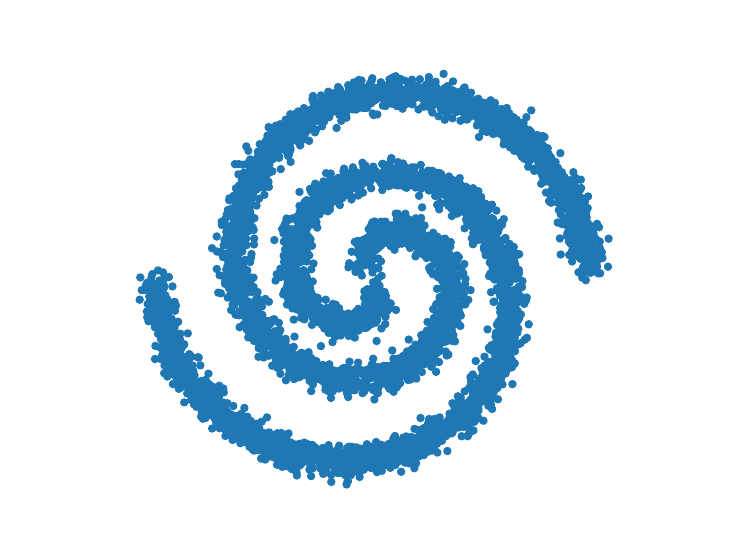}
    \end{minipage}
\hspace{-\samplehinterval}
    \begin{minipage}{\samplempwid\textwidth}
    \centering
    8gaussians
    \includegraphics[width=\textwidth]{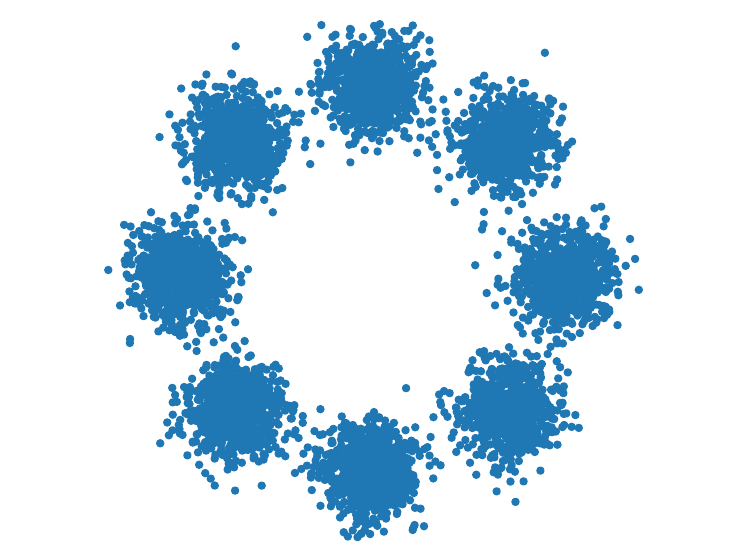}
    \end{minipage}
\hspace{-\samplehinterval}
    \begin{minipage}{\samplempwid\textwidth}
    \centering
    circles
    \includegraphics[width=\textwidth]{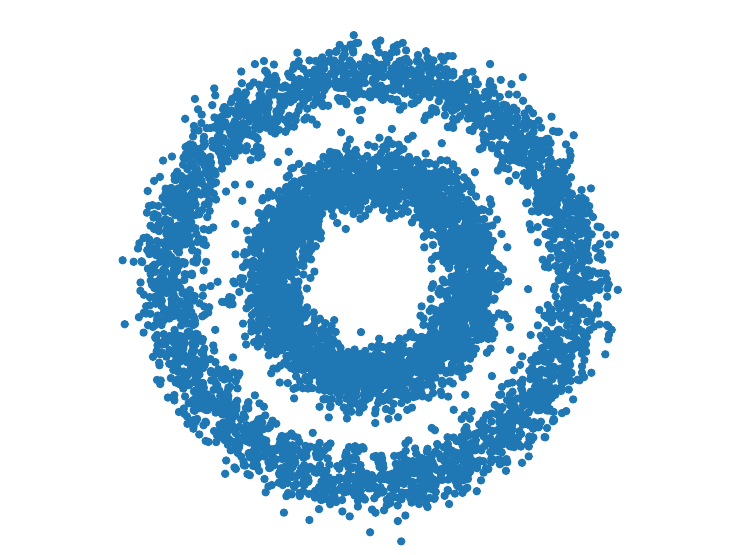}
    \end{minipage}
\hspace{-\samplehinterval}
    \begin{minipage}{\samplempwid\textwidth}
    \centering
    moons
    \includegraphics[width=\textwidth]{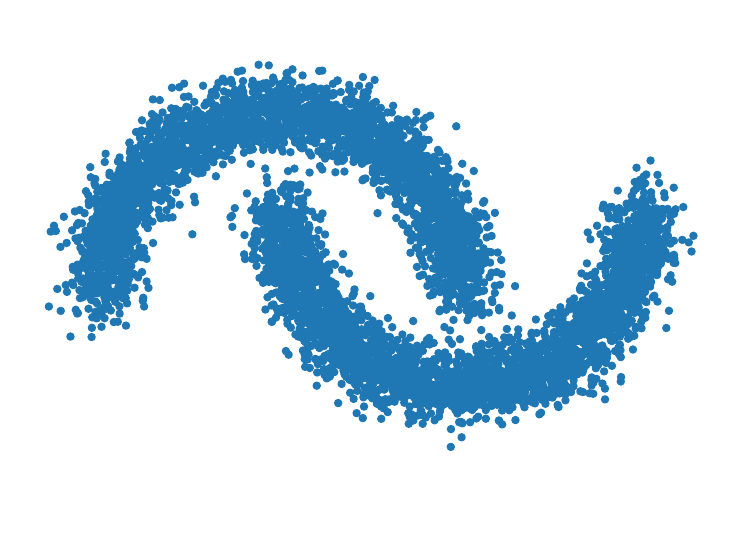}
    \end{minipage}
\hspace{-\samplehinterval}
    \begin{minipage}{\samplempwid\textwidth}
    \centering
    pinwheel
    \includegraphics[width=\textwidth]{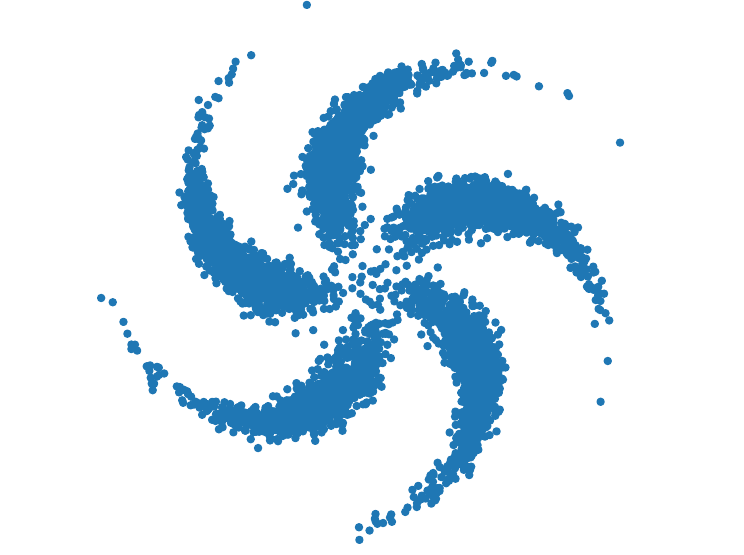}
    \end{minipage}
\hspace{-\samplehinterval}
    \begin{minipage}{\samplempwid\textwidth}
    \centering
    swissroll
    \includegraphics[width=\textwidth]{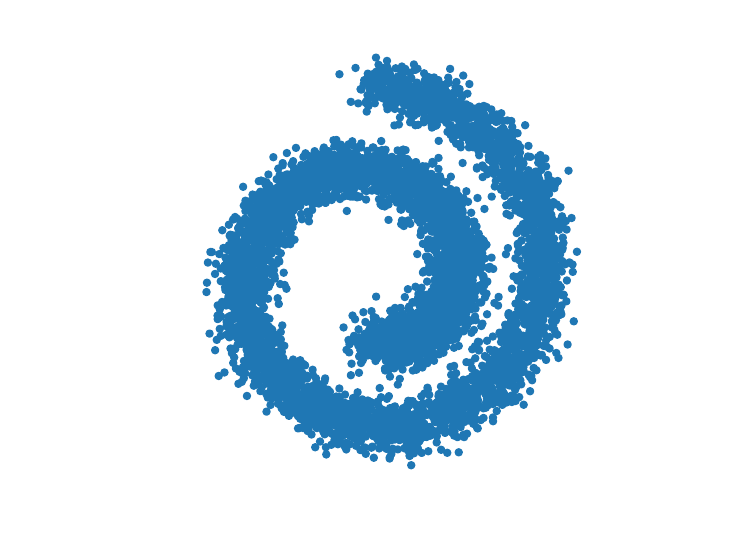}
    \end{minipage}
\hspace{-\samplehinterval}
    \begin{minipage}{\samplempwid\textwidth}
    \centering
    checkerboard
    \includegraphics[width=\textwidth]{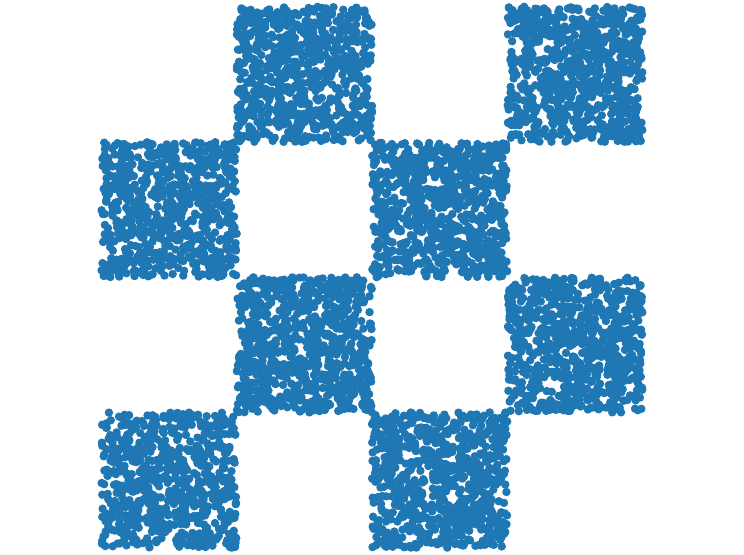}
    \end{minipage}
\hspace{-\samplehinterval}
\end{minipage}
\caption{Visualization of samples for discrete density estimation from ground truth.}
\label{fig:true_synthetic_samples}
\end{figure}

\subsection{Discrete Density Estimation} \label{appendix-sec-discrete-density-estimation}

\textbf{Experimental Details.}
This experiment keeps a consistent setting with \cite{dai2020learning}. We first generate 2D floating-points from a continuous distribution $\hat{p}$ which lacks a closed form but  can be easily sampled. Then, each sample $\hat{\x} := [\hat{\x}_1, \hat{\x}_2] \in \mathbb{R}^2$ is converted to a discrete data point $\x \in \{0,1\}^{32}$ using Gray code. To be specific, given $\hat{\x} \sim \hat{p}$, we quantise both $\hat{\x}_1$ and $\hat{\x}_2$ into $16$-bits binary representations via Gray code \citep{gray1953pulse}, and concatenate them together to obtain a $32$-bits vector $\x$. As a result, the probabilistic mass function in the discrete space is $p(\x) \propto \hat{p} \left( \left[ \operatorname{GrayToFloat}(\x_{1:16}), \operatorname{GrayToFloat}(\x_{17:32}) \right] \right)$. It is noteworthy that learning on this discrete space presents challenges due to the highly non-linear nature of the Gray code transformation.

The energy function is parameterised by a $4$ layer MLP with $256$ hidden dimensions and Swish \citep{ramachandran2017searching} activation. We train the EBM for $10^5$ steps and adopt an Adam optimiser with a learning rate of $0.002$ and a batch size of $128$ to update the parameter. For the energy discrepancy, we choose $w=1, M=32$ for all variants, $\epsilon=0.1$ in ED-Bern, and the window size is $32\times 1$ in ED-Pool. After training, we quantitatively evaluate all methods using the negative log-likelihood (NLL) and the maximum mean discrepancy (MMD). To be specific, the NLL metric is computed based on $4,000$ samples drawn from the data distribution, and the normalisation constant is estimated using importance sampling with $1,000,000$ samples drawn from a variational Bernoulli distribution with $p=0.5$. For the MMD metric, we follow the setting in \cite{zhang2022generative}, which adopts the exponential Hamming kernel with $0.1$ bandwidth. Moreover, the reported performances are averaged over 10 repeated estimations, each with $4,000$ samples, which are drawn from the learned energy function via Gibbs sampling.

\textbf{Qualitative Results.}
We qualitatively visualise the learned energy functions of our proposed approaches in \cref{fig:toy_result_visualisation}. To provide further insights into the oracle energy landscape, we also plot the ground truth samples in Figure \ref{fig:true_synthetic_samples}. The results clearly demonstrate that energy discrepancy effectively fits the data distribution, validating the efficacy of our methods.

\begin{figure}[!t]
\vspace{-5mm}
    \centering
    \includegraphics[width=1.\textwidth]
    {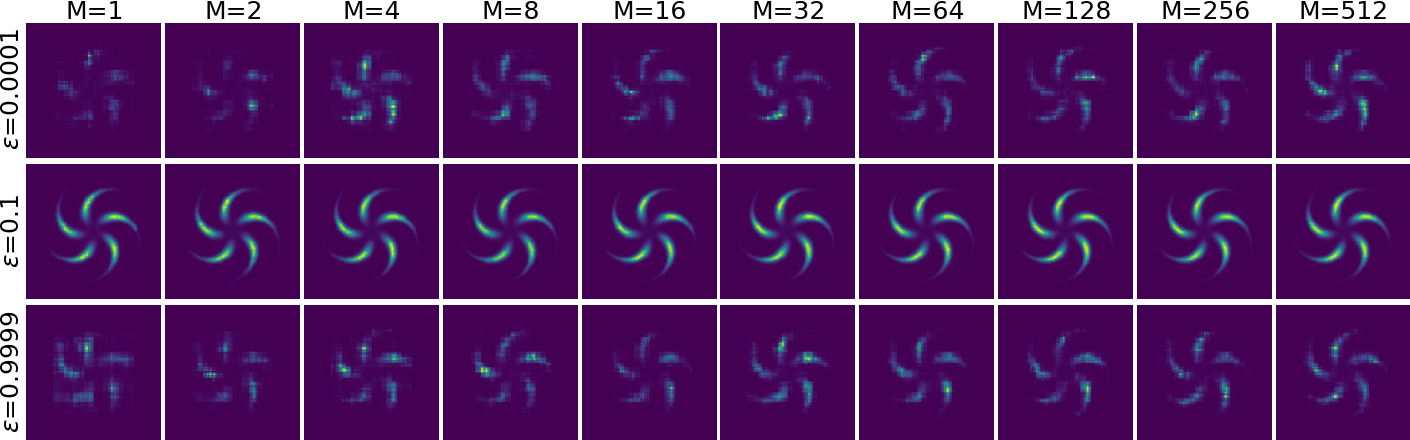}
    \vspace{-6mm}
    \caption{Density estimation results of ED-Bern on the pinwheel with different $\epsilon,M$ and $w=1$.}
    \vspace{-8mm}
    \label{fig:toy-understanding-epsm-appendix}
\end{figure}

\textbf{The Effect of $\epsilon$ in Bernoulli Perturbation.}
Perhaps surprisingly, we find that the proposed energy discrepancy loss with Bernoulli perturbation is very robust to the noise scalar $\epsilon$. 
\begin{wrapfigure}{r}{0.60\linewidth}
\vspace{-4mm}
\centering
\includegraphics[width=.60\textwidth]{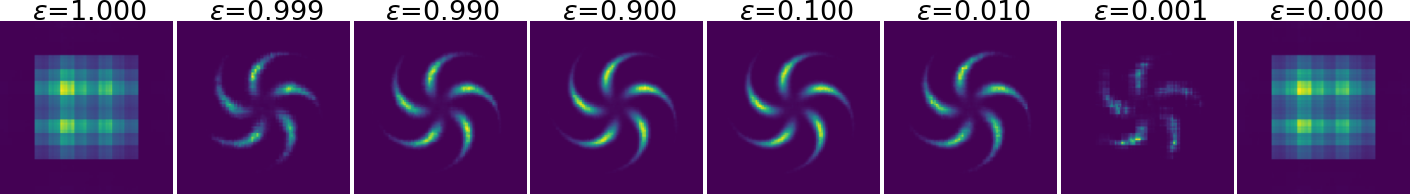}
\vspace{-6mm}
\caption{Density estimation results of ED-Bern on the pinwheel with different $\epsilon$ and $M=32,w=1$.}
\vspace{-2mm}
\label{fig:ed-bern-diff-eps-appendix}
\end{wrapfigure}
In \cref{fig:ed-bern-diff-eps-appendix}, w visualise the learned energy landscapes with different $\epsilon$. 
The results demonstrate that ED-Bern is able to learn faithful energy functions, even with extreme values of $\epsilon$, such as $\epsilon \in \{0.999, 0.001\}$. This highlights the robustness and effectiveness of our approach. In \cref{fig:toy-understanding-epsm-appendix}, we further show that, with $\epsilon \in \{0.9999, 0.0001\}$, ED-Bern can still learn a faithful energy landscape using a large value of $M$. However, when $\epsilon \in \{1, 0\}$, ED-Bern fails to work. It is noteworthy that the choice of $\epsilon$ is highly dependent on the specific structure of the dataset. While ED-Bern exhibits robustness to different values of $\epsilon$ in the synthetic data, we have observed that a large value of $\epsilon$ ($\epsilon \geq 0.1$) is not effective for discrete image modeling. 

\begin{wrapfigure}{r}{0.45\linewidth}
\vspace{-5mm}
\centering
\includegraphics[width=.45\textwidth]{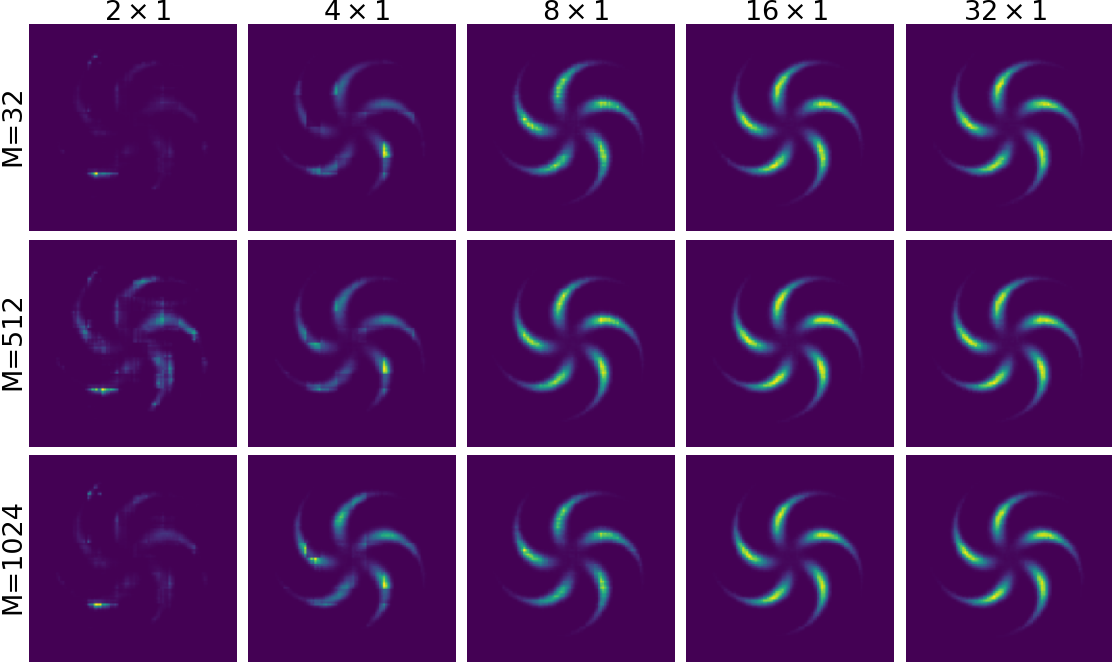}
\vspace{-6mm}
\caption{Density estimation results of ED-Pool on the pinwheel with different window sizes, $M$ and $w=1$.}
\vspace{-4mm}
\label{fig:ed-pool-diff-window-size-appendix}
\end{wrapfigure}
\textbf{The Effect of Window Size in Deterministic Transformation.}
To investigate the effectiveness of the window size in ED-Pool, we conduct experiments in \cref{fig:ed-pool-diff-window-size-appendix} with different window sizes. The results indicate that employing a small window size ({\it e.g.}, $2\times 1$) does not provide sufficient information for energy discrepancy to effectively learn the underlying data structure. Furthermore, our empirical findings suggest that solely increasing the value of $M$ is not a viable solution to address this issue. Again, the choice of the window size should depend on the underlying data structure. In the discrete image modelling, we find that even with a small window size ({\it i.e.}, $4 \times 4$), energy discrepancy yields an energy with low values on the data-support but rapidly diverging values outside of it. Therefore, it fails to learn a faithful energy landscape.

\textbf{Qualitatively Understanding the Effect of $w$ and $M$.}
The hyperparameters $w$ and $M$ play a crucial role in the estimation of energy discrepancy. Increasing $M$ can reduce the variance of the Monte Carlo estimation of the contrastive potential in \eqref{definition-contrastive-potential}, while a proper value of $w$ can improve the stabilisation of training. Here, we evaluate the effect of $w$ and $M$ on the variants of energy discrepancy in \cref{fig:toy-understanding-wm-edbern-appendix,fig:toy-understanding-wm-edpool-appendix,fig:toy-understanding-wm-edgrid-appendix}. Based on empirical observations, we observe that when $w=0$ and $M$ is small ({\it e.g.}, $M \leq 32$ for ED-Bern and $M \leq 64$ for ED-Pool and ED-Grid), energy discrepancy demonstrates rapid divergence and fails to converge. Additionally, we find that increasing $M$ can address this issue to some extent and introducing a non-zero value for $w$ can significantly stabilize the convergence, even with $M=1$. Moreover, larger $w$ tends to produce a flatter estimated energy landscapes, which also aligns with the findings in continuous scenarios of energy discrepancy \cite{schroeder2023energy}.

\begin{figure}[!t]
    \centering
    \includegraphics[width=1.\textwidth]
    {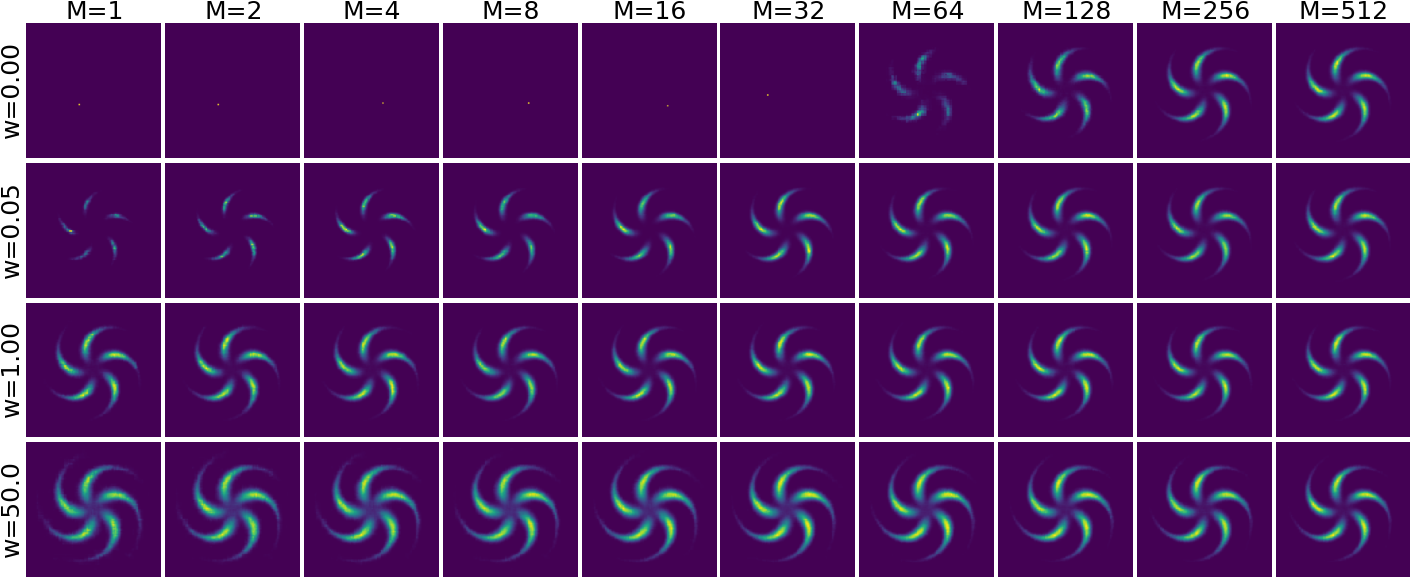}
    \vspace{-4mm}
    \caption{Density estimation results of ED-Bern on the pinwheel with different $w,M$ and $\epsilon=0.1$.}
    \label{fig:toy-understanding-wm-edbern-appendix}
    \vspace{-2mm}
\end{figure}

\begin{figure}[!t]
    \centering
    \includegraphics[width=1.\textwidth]
    {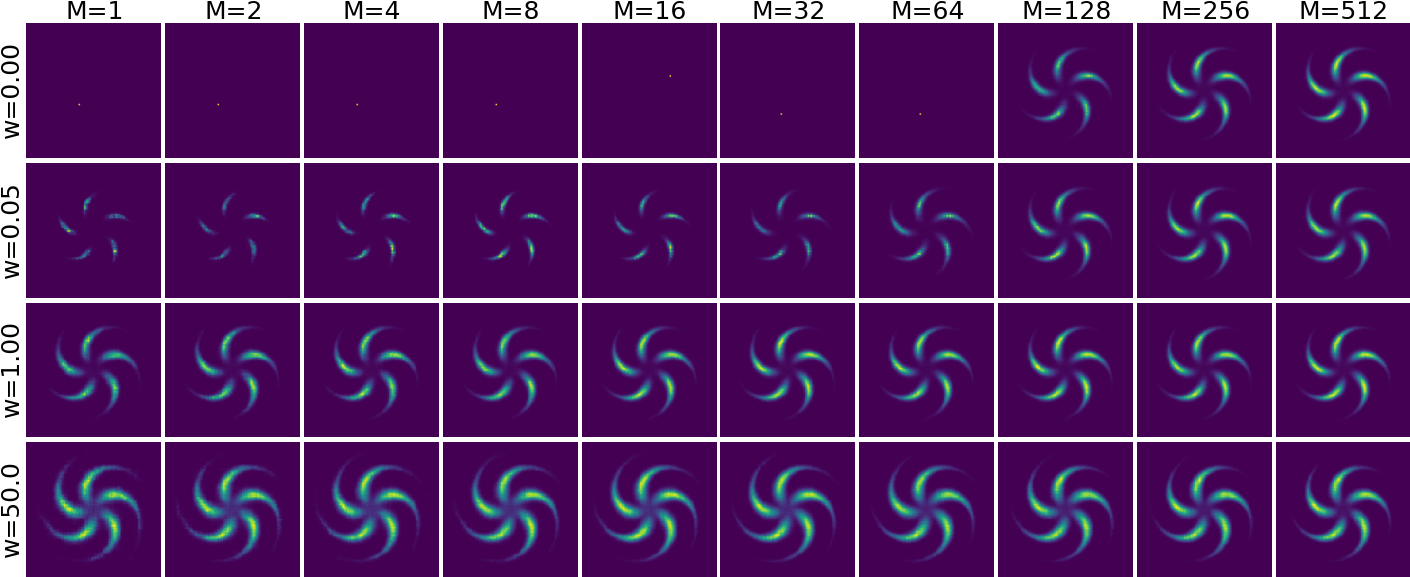}
    \vspace{-4mm}
    \caption{Density estimation results of ED-Pool on the pinwheel with different $w,M$ and the window size is $32 \times 1$.}
    \label{fig:toy-understanding-wm-edpool-appendix}
    \vspace{-2mm}
\end{figure}

\begin{figure}[!t]
    \centering
    \includegraphics[width=1.\textwidth]
    {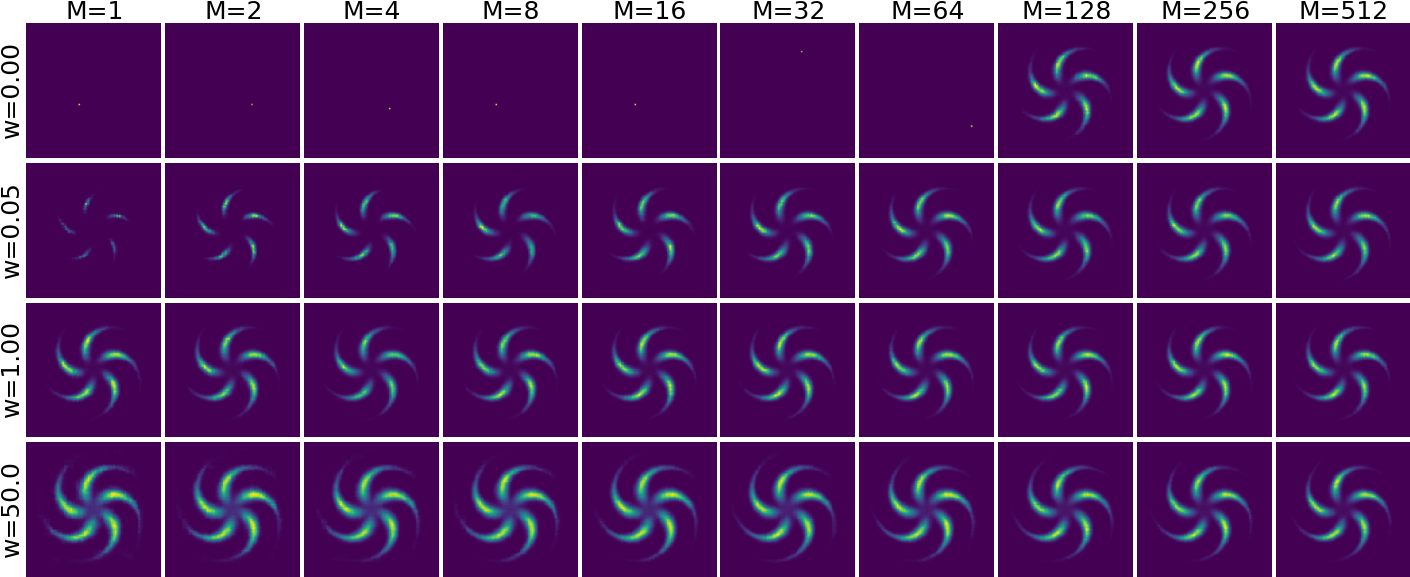}
    \vspace{-4mm}
    \caption{Density estimation results of ED-Grid on the pinwheel with different $w,M$.}
    \label{fig:toy-understanding-wm-edgrid-appendix}
    \vspace{-2mm}
\end{figure}

\subsection{Discrete Image Modelling} \label{appendix-sec-discrete-image-modelling}

\textbf{Experimental Details.}
In this experiment, we parametrise the energy function using ResNet \citep{he2016deep} following the settings in \cite{grathwohl2021oops,zhang2022langevin}, where the network has $8$ residual blocks with $64$ feature maps. Each residual block has $2$ convolutional layers and uses Swish activation function \citep{ramachandran2017searching}. We choose $M=32, w=1$ for all variants of energy discrepancy, $\epsilon=0.001$ for ED-Bern, and the window size is $2\times 2$ for ED-Pool. Note that here we choose a relatively small $\epsilon$ and window size, since we empirically find that the loss of energy discrepancy converges to a constant rapidly with larger $\epsilon$ and window size, which can not provide meaningful gradient information to update the parameters. All models are trained with Adam optimiser with a learning rate of $0.0001$ and a batch size of $100$ for $50,000$ iterations. We perform model evaluation every $5,000$ iterations by conducting Annealed Importance Sampling (AIS) with a discrete Langevin sampler for $10,000$ steps. The reported results are obtained from the model that achieves the best performance on the validation set. After training, we finally report the negative log-likelihood by running $300,000$ iterations of AIS. 

\textbf{Qualitative Results.}
We show the generated images in \cref{fig:sample-ebm-appendix}, which are the samples in the final step of AIS. We see that our methods can generate realistic images on the Omniglot dataset but mediocre images on Caltech Silhouette. We hypothesise that improving the design of the affinity structure in the neighborhood-based transformation can lead to better results. On both the static and dynamic MNIST datasets, ED-Bern and ED-Grid generate diverse and high-quality images. However, ED-Pool experiences mode collapse, resulting in limited variation in the generated samples.

\begin{figure}[!t]
\centering
	\begin{tabular}{cccc}		
	   \includegraphics[width=0.22\textwidth]{figures/image_modelling/ed_bernoulli_stastic_mnist.png}&
        \includegraphics[width=0.22\textwidth]{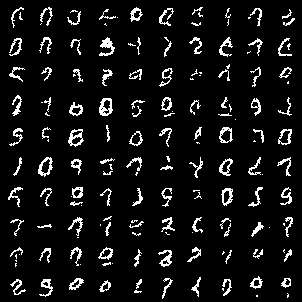}&
        \includegraphics[width=0.22\textwidth]{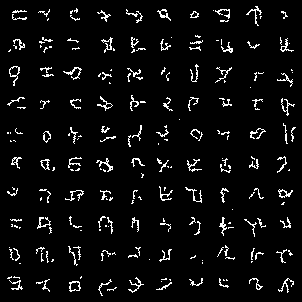}&
        \includegraphics[width=0.22\textwidth]{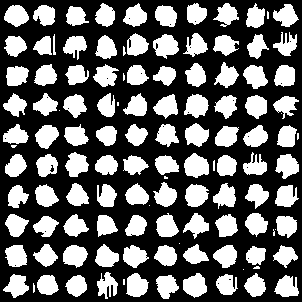} \\
	   \includegraphics[width=0.22\textwidth]{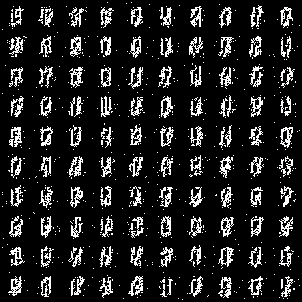}&
        \includegraphics[width=0.22\textwidth]{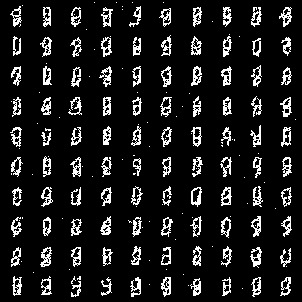}&
        \includegraphics[width=0.22\textwidth]{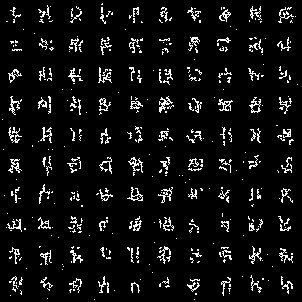}&
        \includegraphics[width=0.22\textwidth]{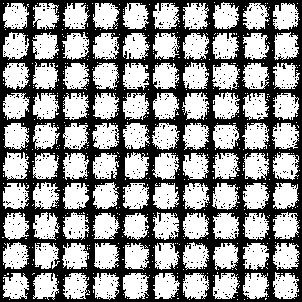} \\
	   \includegraphics[width=0.22\textwidth]{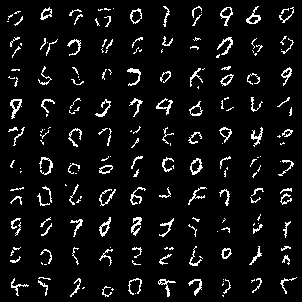}&
        \includegraphics[width=0.22\textwidth]{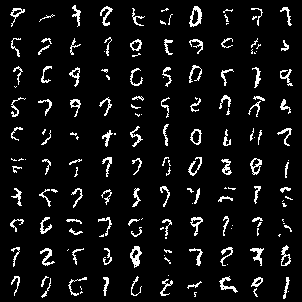}&
        \includegraphics[width=0.22\textwidth]{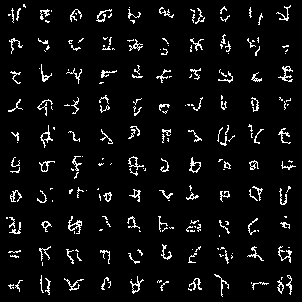}&
        \includegraphics[width=0.22\textwidth]{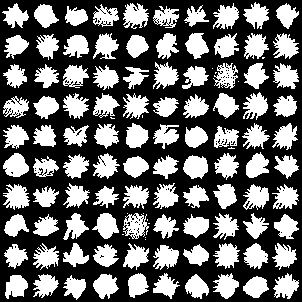} \\
	\end{tabular}
	\caption{Generated samples on discrete image modelling. Left to right: Static MNIST, Dynamic MNIST, Omniglot, Caltech Silhouettes. Top to bottom: ED-Bern, ED-Pool, ED-Grid.}
	\label{fig:sample-ebm-appendix}
\end{figure}

\end{document}